\documentclass{article}
\usepackage{subcaption}
\usepackage{multirow}
\usepackage{algorithm}

\PassOptionsToPackage{numbers, sort&compress}{natbib}

\usepackage[main, final]{neurips_2026}

\usepackage[utf8]{inputenc} 
\usepackage[T1]{fontenc}    
\usepackage{hyperref}       
\usepackage{url}            
\usepackage{booktabs}       
\usepackage{amsfonts}       
\usepackage{nicefrac}       
\usepackage{microtype}      
\usepackage{xcolor}         

\usepackage{hyperref}
\usepackage{algpseudocode}

\usepackage{amsmath}
\usepackage{amssymb}
\usepackage{mathtools}
\usepackage{amsthm}

\usepackage[capitalize,noabbrev]{cleveref}

\theoremstyle{plain}
\newtheorem{theorem}{Theorem}[section]
\newtheorem{proposition}[theorem]{Proposition}

\theoremstyle{definition}

\theoremstyle{remark}

\usepackage[utf8]{inputenc} 
\usepackage[T1]{fontenc}    
\usepackage{hyperref}       
\usepackage{url}            
\usepackage{booktabs}       
\usepackage{amsfonts}       
\usepackage{nicefrac}       
\usepackage{microtype}      
\usepackage{xcolor}         
\usepackage[table]{xcolor}

\definecolor{rankone}{RGB}{255,238,210}   
\definecolor{ranktwo}{RGB}{225,238,255}   
\definecolor{rankthree}{RGB}{255,230,240} 
\title{Advectra: Asymmetric Latent Transport for Non-Stationary Physics}

\author{%
  Nodens Koren \\
  Department of Computer Science\\
  ETH Zürich\\
  Zürich, Switzerland \\
  \texttt{nodens.koren@inf.ethz.ch} \\
  \And
  Thomas Hofmann \\
  Department of Computer Science\\
  ETH Zürich\\
  Zürich, Switzerland \\
  \texttt{thomas.hofmann@inf.ethz.ch} \\
  \AND
  Georgios Kissas \\
  Swiss Data Science Center \\
  ETH Zürich\\  
  Zürich, Switzerland \\
  \texttt{georgios.kissas@sdsc.ethz.ch} \\
}

\begin{document}

\maketitle

\begin{abstract}


Many latent neural operators represent input and output fields in a stationary latent chart. In particular, common latent routing mechanisms use fixed or shared assignment weights for feature projection and reconstruction, limiting their ability to model transport-dominated systems where coherent structures move relative to fixed coordinate frames. We propose \textbf{Advectra}, a transport-aware latent operator that introduces a regularized kinematic coordinate map to decouple source and target coordinate systems. This yields an approximately co-moving latent reference frame and enables asymmetric feature aggregation and reconstruction. Combined with a geometry-aware ordering mechanism for state-space models, Advectra captures advective dynamics while maintaining stable global interactions. Advectra achieves the best performance among evaluated geometry-constrained and form-free baselines on advection-dominated benchmarks, including passive scalar transport in Navier--Stokes flows and Rayleigh--Taylor instability, while demonstrating strong generalization on real-world engineering tasks. These results highlight the benefit of explicit moving-frame structure in neural operators for non-stationary physics.
\end{abstract}

\section{Introduction}
\label{sec:introduction}

While neural operators have emerged as powerful surrogate models for accelerating scientific simulations \citep{lu2021learning, lifourier, wu2024transolver}, many latent neural operator architectures still represent input and output fields through fixed or tied latent charts. Historically, operator learning has often been treated primarily as a challenge of spatial discretization rather than dynamic transport. A large body of recent work has focused on designing architectures that operate directly on unstructured domains, using transformation-based approaches \citep{li2023fourier, ahmad2024diffeomorphic} or graph- and attention-based tokenization \citep{pfaff2020learning, wu2024transolver} to bypass rigid spatial constraints.

This geometric perspective has significantly improved scalability and flexibility, but many latent-routing formulations still encode and decode through a stationary assignment structure. In particular, several physics-attention-style latent operator architectures rely on a shared or tied routing matrix $W$ that projects features into a latent space and reconstructs them using the same assignment structure, $Z = W^\top V$, $Y = W Z'$. This ties aggregation and reconstruction to a single latent chart, imposing a shared reference frame between input and output representation. \looseness=-1

This assumption is fundamentally misaligned with a broad class of physical systems. In transport-dominated phenomena—such as turbulent flows, mixing processes, and bluff-body aerodynamics—the physical state evolves through continuous advection, with mass and momentum migrating across the domain. The relevant structure is therefore not static, but non-stationary: coherent features such as vortices, wakes, and shear layers move, deform, and interact over time. Modeling such systems requires operators that can represent transformations between evolving coordinate frames, rather than mappings defined on a fixed spatial graph.

When applied to these settings, stationary routing mechanisms introduce an inherent structural limitation. Because the same assignment matrix is used for both projection and reconstruction, features are implicitly tied to their original spatial locations. As a result, advected structures must be represented indirectly through repeated local mixing, often leading to excessive diffusion and the loss of sharp transport dynamics. This mismatch becomes particularly pronounced in high-dimensional or highly non-linear regimes, where accurate modeling depends on preserving coherent motion across the domain.

In this work, we propose a \textbf{transport-aware latent operator} that explicitly models non-stationary transformations. Instead of assuming a fixed correspondence between input and output coordinates, we introduce a learned, locally regularized kinematic map on an embedded token-coordinate chart, defining distinct source and target frames. We construct a latent anchor basis coupled to the learned token-coordinate chart, forming an \emph{approximately co-moving latent reference frame} that adapts to the underlying spatiotemporal or geometric structure. By evaluating feature aggregation and reconstruction as asymmetric operations in these transported frames, the operator provides a mechanism for capturing advective motion and evolving spatial structure.

To enable consistent global interactions within this dynamic setting, we further introduce a geometry-aware ordering of the latent representation. Specifically, we align a 1D state-space sequence model (SSM) with the transported anchor geometry. Because SSMs integrate state transitions sequentially, ordering the latent representation along a continuous physical path encourages the recurrence to integrate information across locally adjacent regions of the domain, rather than arbitrary, disconnected permutations of the latent space.

Our contributions are summarized as follows:
\begin{itemize}
    \item \textbf{Transport-Aware Operator Formulation:} We introduce a framework that decouples source and target latent coordinate charts via a learned, locally invertible kinematic map on an embedded token-coordinate chart, enabling neural operators to represent non-stationary transformations through transport-aware source and target frames.
    \item \textbf{Asymmetric Latent Routing:} We replace symmetric routing with a pair of distinct assignment operators, allowing features to be gathered and reconstructed in different coordinate frames without artificial spatial constraints.
    \item \textbf{Geometry-Ordered Latent Dynamics:} We propose an ordering mechanism that aligns latent representations with transported geometry, enabling stable and geometrically coherent sequence modeling using recurrent SSM backbones.
\end{itemize}

Empirically, our method achieves best performance among the evaluated methods across a range of challenging benchmarks, supporting the importance of transport-aware and non-stationary inductive biases for advancing neural operators beyond static spatial approximations.

\section{Background}
\label{sec:background}

\subsection{Neural Operators Beyond Regular Grids}

The foundational goal of neural operators is to learn resolution-invariant mappings between infinite-dimensional function spaces \citep{kovachki2023neural}. Early models such as the Fourier Neural Operator (FNO) \citep{lifourier} achieved strong performance on uniform grids, but real-world physical systems require handling irregular geometries and complex boundary conditions. Consequently, a large body of work has focused on designing architectures that operate directly on unstructured domains.

Graph-based approaches such as the Graph Neural Operator (GNO) \citep{li2020neural} and MeshGraphNet \citep{pfaff2020learning} construct message-passing structures over spatial nodes, while attention-based methods such as OFormer \citep{litransformer} and Transolver \citep{wu2024transolver} treat the domain as a set of unordered tokens. More recently, SpiderSolver \citep{qispidersolver} introduces geometry-aware partitioning by constructing boundary-guided ``spiderweb'' tokens, combining coarse-grained attention across tokens with fine-grained interactions near domain boundaries. \looseness=-1

While these approaches significantly improve flexibility on irregular domains, they share a common structural assumption: the operator is defined over a fixed spatial coordinate frame. Even when the representation is unstructured or adaptively partitioned, the underlying geometry remains stationary. As a result, these methods are well-suited for capturing localized spatial interactions, but are less effective in settings where the relevant structures evolve dynamically over time.

\subsection{Latent Routing and the Limits of Form-Free Operators}

To address the computational challenges of operating on large spatial domains, recent work has introduced latent routing mechanisms. Models such as Transolver employ ``Physics-Attention,'' which projects high-resolution inputs into a smaller set of latent slices, performs attention in this compressed space, and reconstructs the output through the same assignment structure. This design introduces a strong inductive bias by grouping spatial points into structured latent regions.

Recent work has questioned the necessity of such explicit structure. LinearNO \citep{hu2025transolver} shows that symmetric Physics-Attention can be interpreted as a constrained form of generalized linear attention. By removing structured slice-to-slice interactions and relaxing the symmetric weight-sharing constraints between Query and Key projections, LinearNO achieves comparable or improved performance on several benchmarks with reduced parameter complexity.

This result suggests a ``form-free'' perspective on operator learning: that sufficiently expressive attention mechanisms can approximate the desired mappings without explicitly encoding spatial structure. However, such approaches implicitly assume that the operator can be learned without enforcing directional or geometric constraints on the representation.

\subsection{Transport-Dominated Regimes and the Role of Inductive Bias}

While form-free approaches are effective on many standard benchmarks—often characterized by elliptic or diffusion-dominated behavior—we find that this assumption breaks down in transport-dominated regimes. These systems are governed by advective dynamics, where information propagates directionally and the relevant structures evolve continuously across the domain.

In our comparison, we evaluate form-free models such as LinearNO against structured latent operators equipped with physics-aware routing. On highly non-linear, advection-dominated problems—including passive scalar transport in Navier-Stokes flows and forced Kolmogorov turbulence—we observe that structured operators consistently outperform their form-free counterparts.

This performance gap highlights a key limitation of form-free attention: while it can approximate smoothing operators associated with diffusion, it lacks an explicit mechanism for representing directional transport and evolving coordinate frames. At the same time, geometry-aware spatial partitioning methods—such as SpiderSolver—operate directly in the physical coordinate domain and remain tied to a fixed spatial frame, limiting their ability to track moving structures.

Taken together, these observations suggest that neither fully unconstrained nor purely spatially structured approaches are sufficient for complex physical systems. Instead, effective operator learning in transport-dominated regimes requires inductive biases that explicitly account for non-stationarity—capturing both the evolving geometry of the domain and the directional flow of information, necessitating a transition from static grids to dynamically transported, approximately co-moving reference frames. \looseness=-1

\begin{figure}[t]
    \centering
    \includegraphics[width=\textwidth]{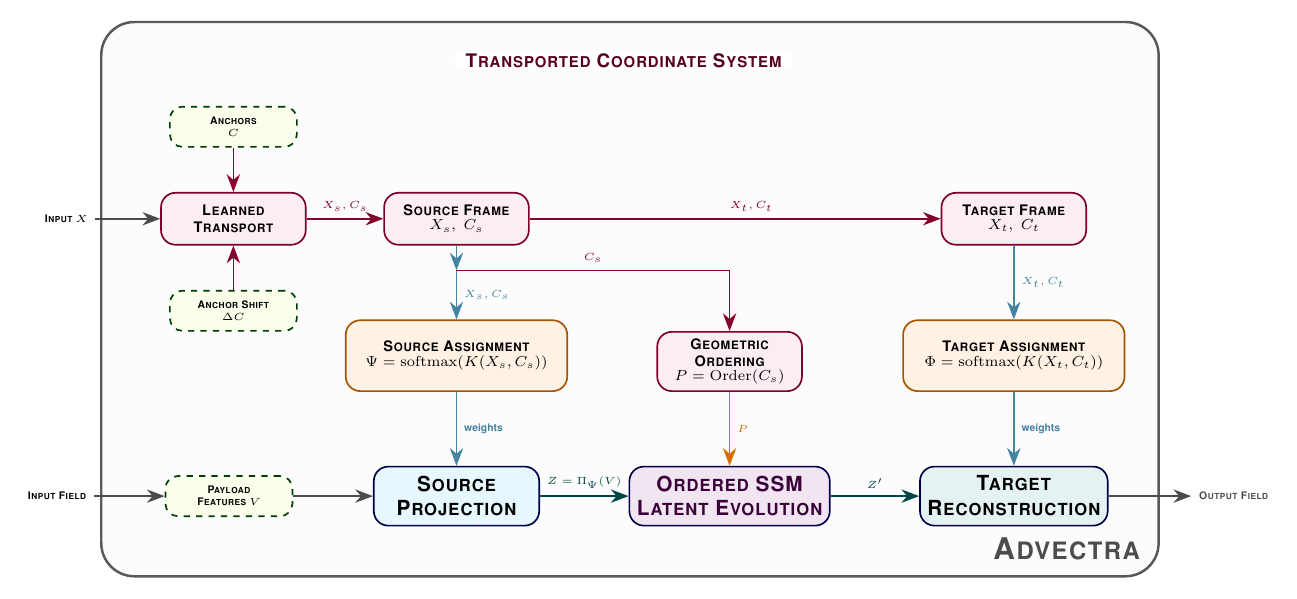}
    \caption{Overview of the Advectra operator.}
    \label{fig:flow_chart}
\end{figure}

\section{Method}
\label{sec:method}

We propose a transport-inspired latent operator for representing non-stationary transformations. The central idea is to decouple feature aggregation and reconstruction through learned source and target coordinate charts, while representing field values via a latent anchor basis. This section presents the high-level mathematical formulation; implementation details such as kernel parameterization, Lipschitz constraint, sorting refresh, and hyperparameters are provided in Appendix~\ref{app:implementations}.


Let \(X_{\mathrm{phys}}\in\mathbb{R}^{M\times d_x}\) denote physical spatial coordinates and \(V_{\mathrm{phys}}\in\mathbb{R}^{M\times d_u}\) denote input features, where \(M\) is the number of spatial points, \(d_x\) is the physical coordinate dimension, and \(d_u\) is the number of input channels. For spatiotemporal problems, let \(t\) denote the prediction time or time offset, broadcast to all spatial points, and form the input token representation
\[
    H_{\mathrm{phys}}=(X_{\mathrm{phys}},V_{\mathrm{phys}},t).
\]
For steady-state or time-independent benchmarks, \(t\) is omitted. We embed the input tokens into a coordinate chart and a feature payload:
\[
    X=\phi_x(H_{\mathrm{phys}})\in\mathbb{R}^{M\times d},
    \qquad
    V=\phi_v(H_{\mathrm{phys}})\in\mathbb{R}^{M\times D},
\]
where \(d\) is the embedded coordinate dimension and \(D\) is the latent feature dimension. Thus, \(X\) is a learned token-coordinate chart used for transport and routing, while \(V\) contains the feature values passed to the latent operator. We introduce \(G\) latent anchors \(C\in\mathbb{R}^{G\times d}\), where \(G\ll M\).

\subsection{Transported Coordinate Systems}
\label{sec:tcs}

We define an anchor-frame context that summarizes both the source anchor chart and the average source-target offset:
\[
    s=\bar C\oplus\bar\Delta,\qquad
    \bar C=\frac{1}{G}\sum_{g=1}^{G}C_g,\qquad
    \bar\Delta=\frac{1}{G}\sum_{g=1}^{G}\Delta_g,
\]
where \(\oplus\) denotes concatenation and \(\Delta\in\mathbb{R}^{G\times d}\) is a learnable anchor-wise target offset initialized at zero. The context \(s\) is an anchor-frame context, not a direct physical-state summary; sample and time dependence enter through the embedded token chart \(X=\phi_x(H_{\mathrm{phys}})\) and feature payload \(V=\phi_v(H_{\mathrm{phys}})\).

Using \(s\), we introduce two learned smooth, Lipschitz-constrained displacement maps,
\[
\mathcal{T}: \mathbb{R}^{d}\times\mathbb{R}^{2d}\to\mathbb{R}^{d},
\qquad
\mathcal{U}: \mathbb{R}^{d}\times\mathbb{R}^{2d}\to\mathbb{R}^{d},
\]
which produce coordinate-dependent displacements for embedded token coordinates and anchors. The constraint yields locally non-folding residual updates; details are given in Appendix~\ref{app:implementations}.

Applying these maps yields transported coordinates
\[
X^{src}=X+\mathcal{T}(X,s),
\qquad
X^{tgt}=X^{src}+\mathcal{T}(X^{src},s),
\]
and transported anchors
\[
C^{src}=C+\mathcal{U}(C,s),
\qquad
C^{tgt}=(C+\Delta)+\mathcal{U}(C+\Delta,s).
\]
We share \(\mathcal{T}\) and \(\mathcal{U}\) across source and target construction to impose a common local displacement rule rather than learning two unrelated coordinate charts. Since the maps are evaluated at different coordinate states, their effective displacements differ between \(X\) and \(X^{src}\), and between \(C\) and \(C+\Delta\), producing distinct source and target charts under this shared transported update.

The \textbf{source chart} \((X^{src},C^{src})\) is used for feature aggregation, and the \textbf{target chart} \((X^{tgt},C^{tgt})\) is used for reconstruction. Because \(X\) is computed from \(H_{\mathrm{phys}}\), the transported charts are time-conditioned on spatiotemporal tasks and geometry-conditioned on steady-state tasks. The learned maps should therefore be interpreted as moving-chart transformations for aligning source and target representations, rather than as direct recovery of sample-wise physical trajectories.

\subsection{Latent Partition-of-Unity Representation}

Feature aggregation is performed via a normalized assignment from spatial points to latent anchors. Let \(\kappa:\mathbb{R}^{d}\times\mathbb{R}^{d}\to\mathbb{R}_{+}\) denote a positive learnable similarity kernel. The source assignment is
\[
\Psi_{ig}
=
\frac{\kappa(x^{src}_i,c^{src}_g)}
{\sum_{k=1}^{G}\kappa(x^{src}_i,c^{src}_k)},
\qquad
\Psi\in\mathbb{R}^{M\times G}.
\]
This defines a partition of unity over anchors for each spatial point. Spatial features are projected into latent variables by
\[
Z_g
=
\frac{\sum_{i=1}^{M}\Psi_{ig}V_i}
{\sum_{i=1}^{M}\Psi_{ig}},
\qquad
Z\in\mathbb{R}^{G\times D}.
\]
We denote this source projection by \(Z=\Pi_\Psi(V)\). The normalization over spatial indices reduces sensitivity to local sampling density.

\subsection{Ordered Latent Dynamics}

Each latent variable is associated with a transported anchor pair. To enable coherent sequential modeling, anchors are ordered according to their transported geometry. We define anchor midpoints
\[
m_g=\frac{1}{2}(C^{src}_g+C^{tgt}_g),
\]
and construct a permutation matrix \(P_\pi\in\{0,1\}^{G\times G}\) using a deterministic greedy nearest-neighbor traversal over \(\{m_g\}\) in the transported coordinate chart. The ordered latent sequence is processed by a bidirectional state-space sequence model:
\[
Z'=\sigma\!\left(P_\pi^\top\,\mathcal{S}_{SSM}(P_\pi Z)\right),
\]
where \(\mathcal{S}_{SSM}\) is the bidirectional SSM operator and \(\sigma\) is a nonlinear activation. This ordering aligns sequential dynamics with geometrically adjacent latent regions, reducing discontinuous interactions from unordered tokens.

\subsection{Asymmetric Reconstruction}

Reconstruction is performed in the target chart using
\[
\Phi_{ig}
=
\frac{\kappa(x^{tgt}_i,c^{tgt}_g)}
{\sum_{k=1}^{G}\kappa(x^{tgt}_i,c^{tgt}_k)},
\qquad
\Phi\in\mathbb{R}^{M\times G}.
\]
The output is reconstructed as
\[
Y_i=\sum_{g=1}^{G}\Phi_{ig}Z'_g,
\qquad
Y\in\mathbb{R}^{M\times D}.
\]
We denote this target decoding operation by \(Y=D_\Phi(Z')\). Unlike stationary operators that reuse a single assignment, Advectra aggregates in the source chart and reconstructs in the target chart, providing an explicit architectural mechanism for non-stationary transformations.

\paragraph{Operator form and decoder interpretation.}
The full Advectra block can be summarized as 
\[
Y
=
D_\Phi\!\left(\mathcal{F}_{\theta,\pi}(\Pi_\Psi(V))\right),
\]
where \(\Psi=\mathcal{A}(X^{src},C^{src})\), \(\Phi=\mathcal{A}(X^{tgt},C^{tgt})\), \(\Pi_\Psi\) is the source projection, \(D_\Phi\) is the target decoder, and \(\mathcal{F}_{\theta,\pi}\) denotes the ordered bidirectional latent SSM dynamics.

This form also connects Advectra to the moving-basis perspective in operator learning. Linear-decoder architectures represent outputs using a fixed basis, $u(y)\approx \sum_{g=1}^{G} z_g \phi_g(y),$
and are therefore limited by how well the target solution manifold can be approximated by a low-dimensional linear subspace. NOMAD~\citep{seidman2022nomad} addresses this limitation with a nonlinear decoder \(g_\theta(z,y)\), enabling the output representation to adapt nonlinearly to the latent code. Advectra provides a complementary form of adaptation: for fixed source and target charts, reconstruction remains linear in the evolved latent coefficients \(Z'\), but the decoder basis is evaluated in a learned target chart through \(\Phi=\mathcal{A}(X^{tgt},C^{tgt})\). Thus, the basis geometry changes through the transported coordinate system rather than through a fully nonlinear code-conditioned decoder. This chart-adaptive view is particularly natural for transport-dominated physics, where the main difficulty is often aligning moving structures between source and target frames.


\section{Experiments}
\subsection{Experimental Setup}


\textbf{PDE Benchmarks.}
Standard neural operator benchmarks (e.g., Darcy flow, Elasticity) are predominantly governed by elliptic or parabolic equations, which exhibit smooth and diffusive behavior. To specifically assess performance on \emph{hyperbolic} and advection-dominated dynamics, we consider three challenging problems from the POSEIDON benchmark \citep{herde2024POSEIDON}: (i) 2D compressible Euler equations with Rayleigh--Taylor instability (GCE-RT), (ii) 2D Navier--Stokes with passive tracer transport (NS-Tracer-PwC), and (iii) 2D forced Navier--Stokes Kolmogorov flow (FNS-KF). These problems involve strong transport effects, including sharp gradients, filamentary structures, and chaotic mixing. We additionally include the standard Airfoil benchmark \citep{li2023fourier} for comparison with prior work \citep{wu2024transolver, tiwarilatent}.

For NS-Tracer-PwC and FNS-KF, we follow the POSEIDON protocol and train under reduced spatial and temporal resolution via subsampling by a factor of two. We report results at $t=7$ for GCE-RT and at $t=14$ for NS-Tracer-PwC and FNS-KF, corresponding to long-horizon prediction settings.

\textbf{Industrial Design Benchmarks.}
To evaluate scalability and performance on complex geometries, we consider AirfRANS \citep{bonnet2022airfrans}, AirCraft \citep{luotransolver++}, and ShapeNet-Car \citep{chang2015shapenet}. These datasets involve unstructured meshes and high-dimensional 3D flows. In addition to field-wise relative $L_2$ errors, we report Spearman’s rank correlation of aerodynamic coefficients (lift $C_L$ and drag $C_D$) for AirfRANS and ShapeNet-Car to assess the reliability of predicted design rankings.

\textbf{Baselines.}
We compare against a comprehensive set of baseline methods spanning spectral operators (e.g., FNO, U-NO), convolutional and graph-based models, attention-based architectures, and recent latent-space approaches such as Transolver \citep{wu2024transolver}, Transolver++ \citep{luotransolver++}, and LaMO \citep{tiwarilatent}. This selection covers both frequency-domain and physical-space paradigms, as well as different mechanisms for latent interaction and representation learning. Further implementation details and experimental settings are provided in Appendix~\ref{app:implementations}. 

\subsection{Results on PDE Benchmarks}
\label{sec:pde_results}

Table~\ref{tab:pde_results} presents quantitative results across a diverse set of PDE benchmarks, spanning both transport-dominated regimes and globally smooth flow dynamics. Overall, we observe a clear distinction between spectral and physical-space models, which can be explained through the intrinsic structure of the underlying PDEs.

\begin{table*}[t]
\centering
\caption{Performance comparison across different PDE benchmarks. Top-2 results in each column are highlighted: best (orange) and second (blue).}
\label{tab:pde_results}
\begin{sc}
\makebox[\textwidth][c]{%
\resizebox{0.94\textwidth}{!}{%
\begin{tabular}{l|cccc|cc|c|c}
\toprule
\multicolumn{1}{c|}{\multirow{2}{*}{\textbf{Model}}} & \multicolumn{4}{c|}{\textbf{GCE-RT}} & \multicolumn{2}{c|}{\textbf{NS-Tracer-PwC}} & \multicolumn{1}{c|}{\textbf{FNS-KF}} & \multicolumn{1}{c}{\textbf{Airfoil}} \\ 
\cmidrule(lr){2-5} \cmidrule(lr){6-7} \cmidrule(lr){8-8} \cmidrule(lr){9-9} 
\multicolumn{1}{c|}{} & $\rho$ $\downarrow$ & $\mathbf{v}$ $\downarrow$ & $p$ $\downarrow$ & tracer $\downarrow$ & $\mathbf{v}$ $\downarrow$ & tracer $\downarrow$ & $\mathbf{v}$ $\downarrow$ & $\mathbf{M}$ $\downarrow$ \\ 
\midrule

U-Net \citep{ronneberger2015u} & 0.0334 & 0.1269 & 0.0039 & 0.1057 & 0.2062 & \cellcolor{ranktwo}0.2047 & 0.1934 & 0.0079 \\ 
NOMAD \citep{seidman2022nomad} & 0.1844 & 0.7386 & 0.0940 & 0.3751 & 0.7755 & 0.6398 & 0.7782 & 0.0379 \\ 

\midrule

FNO \citep{lifourier} & 0.0379 & 0.1346 & 0.0057 & 0.1200 & \cellcolor{ranktwo}0.1577 & 0.3426 & \cellcolor{ranktwo}0.1317 & 0.0138 \\
F-FNO \citep{tranfactorized} & 0.0648 & 0.2301 & 0.0149 & 0.1941 & 0.3972 & 0.5349 & 0.3727 & 0.0078 \\
U-NO \citep{rahmanu} & 0.0352 & 0.1105 & 0.0075 & 0.1065 & 0.1604 & 0.2488 & 0.1538 & 0.0078 \\
LSM \citep{wu2023solving} & 0.0421 & 0.1584 & 0.0064 & 0.1326 & 0.2402 & 0.2462 & 0.2199 & 0.0059 \\ 
\midrule

Galerkin \citep{cao2021choose} & 0.0238 & 0.0946 & 0.0024 & 0.0765 & 0.6327 & 0.5550 & 0.5347 & 0.0118 \\
ONO \citep{xiao2024improved} & 0.0764 & 0.1915 & 0.0228 & 0.2208 & 0.5442 & 0.4995 & 0.5530 & 0.0053 \\
Transolver \citep{wu2024transolver} & 0.0185 & 0.0737 & 0.0017 & 0.0605 & 0.1796 & 0.2335 & 0.2632 & 0.0052 \\
Transolver++ \citep{luotransolver++} & \cellcolor{ranktwo}0.0171 & \cellcolor{ranktwo}0.0671 & 0.0018 & \cellcolor{ranktwo}0.0551 & 0.3559 & 0.3882 & 0.4125 & 0.0048 \\
LinearNO \citep{hu2025transolver} & 0.0173 & 0.0683 & \cellcolor{ranktwo}0.0015 & 0.0562 & 0.2326 & 0.2811 & 0.3015 & 0.0049 \\ 
\midrule

LaMO \citep{tiwarilatent} & 0.0179 & 0.0700 & 0.0017 & 0.0584 & 0.1773 & 0.2311 & 0.1627 & \cellcolor{ranktwo}0.0047 \\ 
\midrule

\textbf{Advectra (ours)} & \cellcolor{rankone}\textbf{0.0149} & \cellcolor{rankone}\textbf{0.0595} & \cellcolor{rankone}\textbf{0.0014} & \cellcolor{rankone}\textbf{0.0488} & \cellcolor{rankone}\textbf{0.1209} & \cellcolor{rankone}\textbf{0.1607} & \cellcolor{rankone}\textbf{0.1202} & \cellcolor{rankone}\textbf{0.0041} \\ 
\bottomrule
\end{tabular}%
}%
}
\end{sc}
\vspace{-1em}
\end{table*}

\textbf{Spectral vs Physical-Space Behavior.}
On NS-Tracer-PwC and FNS-KF, spectral models such as FNO achieve strong performance on the velocity fields, even outperforming several transformer-based baselines. This is consistent with the fact that velocity fields in these problems tend to be relatively smooth and globally correlated, making them well-suited to Fourier-based representations.

However, this advantage does not extend to transported quantities. In particular, FNO performs significantly worse on the tracer field, which exhibits sharp gradients and filamentary structures induced by advection. These non-smooth features violate the global smoothness assumptions of spectral methods, leading to substantial degradation in performance. U-NO partially alleviates this limitation through its multi-resolution design, which introduces localized representations and reduces the performance gap on tracer prediction. This suggests that \emph{localization is important for capturing advective transport}, although multi-resolution alone is insufficient to fully resolve complex transport dynamics. \looseness=-1

We also note that NS-Tracer-PwC and FNS-KF are trained under reduced spatial and temporal resolution due to subsampling. Spectral models such as FNO are known to be robust in such settings by leveraging low-frequency structure, which helps explain their strong performance on velocity fields. In contrast, Advectra maintains strong performance without relying on global smoothness assumptions, instead capturing transport dynamics directly, enabling coherent predictions even with limited spatial detail.

On the GCE-RT dataset, which is dominated by highly chaotic instabilities and strong local interactions, physical-space models consistently outperform spectral approaches. The Rayleigh--Taylor instability induces multi-scale mixing and rapidly evolving local structures, where global spectral representations become less effective. \emph{This highlights a key distinction: spectral models excel in smooth regimes, whereas physical-space models are better suited for transport-dominated, non-smooth dynamics.}

\textbf{Latent Interaction and Representation Structure.}
We further investigate the importance of latent interaction through LinearNO, which replaces slice-to-slice attention with a simplified linear operation. While efficient, this design performs poorly on transport-heavy datasets such as NS-Tracer-PwC and FNS-KF. The degradation (e.g., tracer error of $0.2811$ vs. $0.2335$ for Transolver) suggests that \emph{explicit interaction among latent slices is important for advective dynamics}.

We also observe that models with decoupled projections for slice assignment and feature aggregation, such as Transolver and LaMO, tend to outperform coupled-projection variants such as Transolver++ and LinearNO on these transport-heavy benchmarks.

\textbf{Advectra Performance Gains.}
Advectra consistently achieves the best performance across the evaluated benchmarks, demonstrating strong generalization across both spectral-friendly and transport-dominated regimes. Notably, on NS-Tracer-PwC, Advectra reduces the tracer error to $0.1607$, corresponding to a $30.4\%$ relative improvement over the strongest slicing-based baseline, LaMO. Similar gains are observed on FNS-KF, where Advectra achieves over a $26\%$ relative error reduction on velocity prediction compared to LaMO.

These results highlight the importance of explicitly accounting for transport dynamics, suggesting that \emph{Advectra captures global structure while better preserving local transport-dominated features, enabling robust performance across diverse PDE regimes.}

\begin{table*}[t]
\centering
\caption{Quantitative results on industrial benchmarks AirfRANS and ShapeNet-Car. We report relative $L_2$ error ($\downarrow$) and Spearman's rank correlation ($\uparrow$). Top-2 results in each column are highlighted: best (orange) and second (blue).}
\label{tab:industrial_results}
\begin{sc}
\resizebox{0.95\textwidth}{!}{%
\begin{tabular}{l|cccc|cccc}
\toprule
\multicolumn{1}{c|}{\multirow{2}{*}{\textbf{Model}}} 
& \multicolumn{4}{c|}{\textbf{AirfRANS}} 
& \multicolumn{4}{c}{\textbf{ShapeNet-Car}} \\
\cmidrule(lr){2-5} \cmidrule(lr){6-9} 
\multicolumn{1}{c|}{} 
& Volume $\downarrow$ & Surf $\downarrow$ & $C_L$ $\downarrow$ & $\rho_L$ $\uparrow$ 
& Volume $\downarrow$ & Surf $\downarrow$ & $C_D$ $\downarrow$ & $\rho_D$ $\uparrow$ \\ 
\midrule

GraphSAGE \citep{hamilton2017inductive} & 0.0087 & 0.0184 & 0.1476 & 0.9964 & 0.0461 & 0.1050 & 0.0270 & 0.9695 \\
PointNet \citep{qi2017pointnet} & 0.0253 & 0.0996 & 0.1973 & 0.9919 & 0.0494 & 0.1104 & 0.0298 & 0.9583 \\
Graph U-Net \citep{gao2019graph} & 0.0076 & 0.0144 & 0.1677 & 0.9949 & 0.0471 & 0.1102 & 0.0226 & 0.9725 \\
MeshGraphNet \citep{pfaff2020learning} & 0.0214 & 0.0387 & 0.2252 & 0.9945 & 0.0354 & 0.0781 & 0.0168 & 0.9840 \\
GNO \citep{li2020neural} & 0.0269 & 0.0405 & 0.2016 & 0.9938 & 0.0383 & 0.0815 & 0.0172 & 0.9834 \\
Galerkin \citep{cao2021choose} & 0.0074 & 0.0159 & 0.2336 & 0.9951 & 0.0339 & 0.0878 & 0.0179 & 0.9764 \\
Geo-FNO \citep{li2023fourier} & 0.0361 & 0.0301 & 0.6161 & 0.9257 & 0.1670 & 0.2378 & 0.0664 & 0.8280 \\
GNOT \citep{hao2023gnot} & 0.0049 & 0.0152 & 0.1992 & 0.9942 & 0.0329 & 0.0798 & 0.0178 & 0.9833 \\
GINO \citep{li2023geometry} & 0.0297 & 0.0482 & 0.1821 & 0.9958 & 0.0386 & 0.0810 & 0.0184 & 0.9826 \\
3D-GeoCA \citep{deng2024geometry} & - & - & - & - & 0.0319 & 0.0779 & 0.0159 & 0.9842 \\
Transolver \citep{wu2024transolver} & 0.0037 & 0.0142 & 0.1030 & 0.9978 & 0.0225 & 0.0789 & 0.0123 & 0.9915 \\
Transolver++ \citep{luotransolver++} & 0.0025 & 0.0084 & 0.1228 & 0.9976 & 0.0226 & 0.0800 & 0.0132 & 0.9914 \\
LaMO \citep{tiwarilatent} & 0.0020 & 0.0066 & 0.0977 & 0.9983 & 0.0212 & 0.0763 & 0.0119 & 0.9907 \\
SpiderSolver \citep{qispidersolver} & 0.0017 & \cellcolor{ranktwo}0.0043 & 0.0741 & 0.9988 & 0.0210 & \cellcolor{ranktwo}0.0738 & \cellcolor{ranktwo}0.0100 & \cellcolor{ranktwo}0.9928 \\
LinearNO \citep{hu2025transolver} & \cellcolor{ranktwo}0.0011 & 0.0077 & \cellcolor{ranktwo}0.0491 & \cellcolor{ranktwo}0.9992 & \cellcolor{ranktwo}0.0194 & 0.0754 & 0.0106 & 0.9925 \\ 
\midrule

\textbf{Advectra (Ours)} 
& \cellcolor{rankone}\textbf{0.0006} 
& \cellcolor{rankone}\textbf{0.0041} 
& \cellcolor{rankone}\textbf{0.0481} 
& \cellcolor{rankone}\textbf{0.9995} 
& \cellcolor{rankone}\textbf{0.0149} 
& \cellcolor{rankone}\textbf{0.0716} 
& \cellcolor{rankone}\textbf{0.0094} 
& \cellcolor{rankone}\textbf{0.9941} \\

\bottomrule
\end{tabular}%
}
\end{sc}
\end{table*}

\subsection{Results on Industrial Design Benchmarks}
\label{sec:industrial_results}

Tables \ref{tab:industrial_results} and \ref{tab:aircraft_results} summarize the quantitative performance on large-scale industrial benchmarks, with a primary focus on complex 3D external aerodynamics (ShapeNet-Car and AirCraft).

\textbf{Scaling SSMs to 3D Geometries.}
While latent State-Space Models (SSMs), such as LaMO \citep{tiwarilatent}, demonstrate strong performance on 2D PDE benchmarks, we observe that this advantage diminishes when scaling to complex 3D geometries. In particular, on the AirCraft benchmark (Table \ref{tab:aircraft_results}), LaMO no longer consistently outperforms Transformer-based baselines such as Transolver and LinearNO. This suggests that the effectiveness of SSM-based latent sequence modeling degrades in higher-dimensional, irregular domains.

A key challenge lies in constructing a meaningful 1D latent sequence from 3D spatial structures. Compared to 2D settings, where spatial locality can be more naturally preserved, projecting 3D geometries into a sequential representation introduces significantly higher distortion. \emph{This indicates that ordering becomes increasingly lossy as dimensionality increases}, making it difficult for SSMs to maintain coherent long-range structure.

Advectra mitigates this issue by introducing a geometry-aware ordering mechanism that explicitly organizes latent tokens based on spatial proximity. This produces a more stable and structured latent sequence for the SSM core, improving robustness in highly irregular 3D domains. While Advectra also incorporates transport modeling, we observe that the primary gains in these industrial benchmarks are associated with improved geometric organization. We further validate this observation through ablation studies in Sec.~\ref{sec:ablation}.

\textbf{3D Field Prediction.}
With this stabilized latent structure, Advectra achieves best performance among the evaluated methods across the 3D industrial datasets. On ShapeNet-Car, Advectra reduces the volume $L_2$ error to $0.0149$, outperforming both LinearNO ($0.0194$) and LaMO ($0.0212$). The performance gap is even more pronounced on the AirCraft benchmark, which requires the simultaneous prediction of multiple physical fields. While baseline dominance is fragmented—e.g., Transolver performs well on certain velocity components and LinearNO on others—Advectra consistently achieves the lowest error across all variables. Notably, it attains a $37.3\%$ relative improvement on the $U$-velocity field ($0.0057$) over the strongest baseline.

\textbf{Utility for Engineering Design.}
Beyond point-wise accuracy, practical engineering workflows require reliable ranking of candidate designs. Advectra demonstrates strong utility in this setting, achieving the highest Spearman’s rank correlations for macroscopic aerodynamic quantities. On ShapeNet-Car, Advectra achieves near-perfect drag correlation ($\rho_D = 0.9941$), while on AirfRANS it similarly attains the best lift correlation ($\rho_L = 0.9995$). \emph{These results indicate that Advectra not only improves field prediction accuracy but also improves rank-correlation metrics that are relevant for downstream design optimization.}







\begin{table*}[h!]
\centering
\caption{Ablation study of Advectra. We evaluate Geometry-Aware Ordering, Latent Transport, and an independent-warp variant that uses separate source/target displacement maps instead of the shared transported update. Top-2 results per column are highlighted: best (orange) and second (blue).}
\label{tab:ablation_study}
\begin{sc}
\resizebox{0.84\textwidth}{!}{%
\begin{tabular}{l|cc|cccc}
\toprule
\multicolumn{1}{c|}{\multirow{2}{*}{\textbf{Configuration}}} 
& \multicolumn{2}{c|}{\textbf{NS-Tracer-PwC}} 
& \multicolumn{4}{c}{\textbf{ShapeNet-Car}} \\ 
\cmidrule(lr){2-3} \cmidrule(lr){4-7} 
\multicolumn{1}{c|}{} 
& $\mathbf{v}$ $\downarrow$ & tracer $\downarrow$ 
& Volume $\downarrow$ & Surf $\downarrow$ & $C_D$ $\downarrow$ & $\rho_D$ $\uparrow$ \\ 
\midrule

Advectra (Full) 
& \cellcolor{rankone}\textbf{0.1209} 
& \cellcolor{rankone}\textbf{0.1607} 
& \cellcolor{rankone}\textbf{0.0149} 
& \cellcolor{rankone}\textbf{0.0716} 
& \cellcolor{rankone}\textbf{0.0094} 
& \cellcolor{rankone}\textbf{0.9941} \\ 
\midrule

\quad Independent Warps
& 0.1548
& 0.2183
& 0.0174
& \cellcolor{ranktwo}0.0732
& \cellcolor{ranktwo}0.0101
& \cellcolor{ranktwo}0.9931 \\

\quad w/o Ordering
& \cellcolor{ranktwo}0.1446 
& \cellcolor{ranktwo}0.1840 
& 0.0169 
& 0.0766 
& 0.0113 
& 0.9923 \\

\quad w/o Transport 
& 0.1858 
& 0.2435 
& \cellcolor{ranktwo}0.0159 
& 0.0742 
& 0.0108 
& 0.9928 \\

\quad w/o Both 
& 0.2001 
& 0.2525 
& 0.0169 
& 0.0797 
& 0.0129 
& 0.9918 \\ 

\bottomrule
\end{tabular}%
}
\end{sc}
\end{table*}

\section{Analysis of Transport and Ordering}
\label{sec:ablation}

To isolate the contributions of Advectra's core mechanisms, we ablate geometry-aware ordering, latent transport, and their combination. We also test an \emph{Independent Warps} variant, which replaces the shared transported update with separate learned source and target displacement maps. This variant is more flexible than Advectra and tests whether generic coordinate perturbation alone explains the gains. Table~\ref{tab:ablation_study} reports results on NS-Tracer-PwC and ShapeNet-Car.

\textbf{Transient Flows Benefit Most from Transport.}
On the advection-dominated NS-Tracer-PwC dataset, removing transport causes the largest degradation, increasing tracer error by \(51.5\%\) (\(0.1607 \to 0.2435\)). Independent Warps improves over removing transport but remains substantially worse than Advectra (\(0.2183\) vs. \(0.1607\)), suggesting that generic source/target coordinate perturbations alone do not explain the gain. Removing ordering produces a smaller but still notable degradation (\(0.1607 \to 0.1840\)). \emph{These results show that transport is the dominant architectural factor in this transport-dominated regime, while ordering provides complementary structural regularization.}

\textbf{3D Geometries Benefit Most from Geometric Ordering.}
On the steady-state 3D ShapeNet-Car benchmark, the largest drop comes from removing geometry-aware ordering: volume error increases from \(0.0149\) to \(0.0169\), exceeding the degradation from removing transport. Independent Warps gives mixed results, slightly improving some surface/coefficient metrics over no transport while worsening volume error. This suggests that arbitrary coordinate deformation is not consistently beneficial for complex 3D geometries. Without ordering, the latent sequence loses geometric coherence, causing the SSM core to process distant regions as sequence neighbors. \emph{These results identify geometric ordering as the dominant architectural factor for preserving spatial coherence in complex 3D geometry regimes.}

\textbf{Problem-Dependent Architectural Priors.}
The asymmetric impact of these ablations shows that each mechanism matters in different regimes. For transient, transport-heavy PDEs, transport-aware coordinate decoupling is the main driver; for complex 3D boundary-value problems, maintaining geometric consistency under dimensional compression becomes the dominant bottleneck. The Independent Warps result further suggests that Advectra's gain is not merely due to increased coordinate-warping flexibility. \emph{Advectra combines transport and ordering, allowing the architecture to adapt across regimes where either non-stationary motion or geometric compression is limiting.}

\section{Conclusion}

We introduced Advectra, a transport-aware latent operator designed to model non-stationary physical systems. Unlike conventional neural operators that assume a fixed spatial correspondence between input and output domains, Advectra explicitly decouples coordinate systems through a learned transport-inspired coordinate displacement field. This enables asymmetric feature aggregation and reconstruction in distinct frames, allowing the model to represent advective motion more explicitly than symmetric routing mechanisms, which must approximate such motion through repeated local mixing. In addition, we incorporate a geometry-aware ordering mechanism that aligns latent representations with spatial structure, enabling stable and physically meaningful sequence modeling with state-space backbones.

Across a diverse set of benchmarks, including transport-dominated PDEs and large-scale industrial geometries, Advectra consistently outperforms both form-free and geometry-constrained baselines. Our results indicate that explicit transport-aware structure is beneficial for capturing non-stationary dynamics, particularly in settings involving sharp gradients, coherent structures, and long-range advection. At the same time, we show that preserving geometric consistency becomes increasingly important in higher-dimensional domains, where naive sequence representations can severely distort spatial relationships. The ablation study further highlights that these two inductive biases—transport and ordering—play complementary roles depending on the underlying physical regime.

Overall, our findings suggest that advancing neural operators beyond static spatial mappings requires a shift toward representations that account for evolving coordinate systems and directional information flow. Advectra provides a unified framework that integrates these principles, enabling robust performance across both transient and steady-state settings. We believe this perspective opens new directions for operator learning, particularly for problems involving multi-scale transport, complex geometries, and long-horizon prediction.

\newpage

\bibliographystyle{unsrtnat}
\bibliography{example_paper}

\newpage
\appendix

\begin{figure*}[h!]
    \centering
    \includegraphics[width=\textwidth]{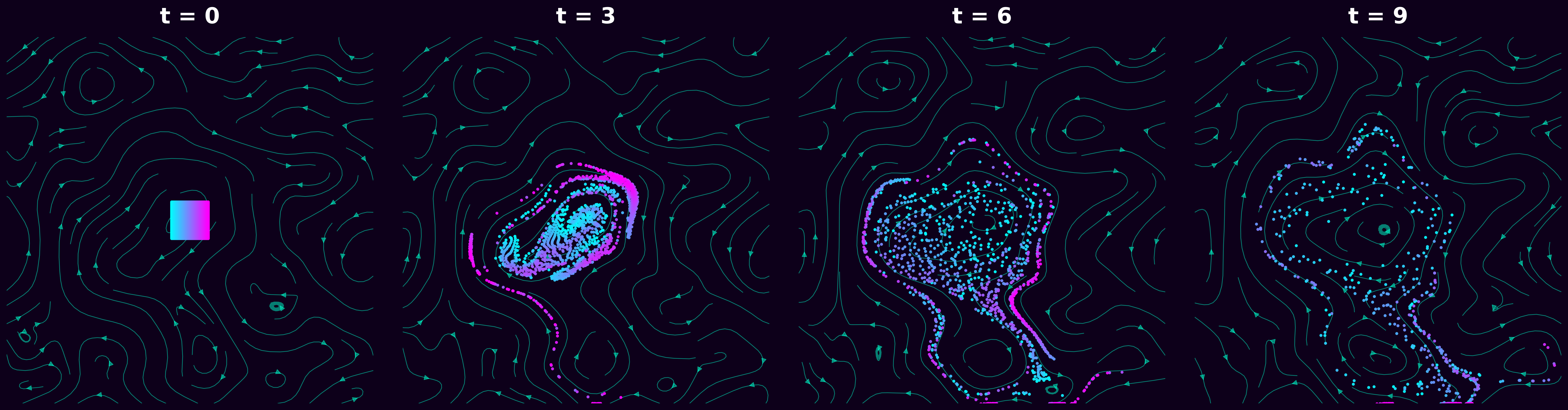}
    \caption{Visualizing the contrast between Eulerian and Lagrangian perspectives in Navier-Stokes flow. The \textbf{Eulerian view} treats the domain as a fixed grid, measuring fluid velocity at specific, stationary locations (cyan streamlines). In contrast, the \textbf{Lagrangian view} follows the individual trajectories of material parcels as they move through space. Here, a dense block of tracers (colored by initial position) acts as a set of ``moving observers.'' As $t$ increases from $0$ to $9$, these tracers are advected by the unsteady flow; their severe stretching and folding demonstrates how the Lagrangian frame of reference deforms over time, physically reconfiguring the spatial relationships between different regions of the fluid.}
    \label{fig:transport_4panel}
\end{figure*}

\begin{table}[H]
\centering
\caption{Quantitative results on the \textbf{AirCraft} benchmark (3D external aerodynamics). We compare against geometric baselines and neural operators. We report relative $L_2$ error ($\downarrow$). Top-2 results in each column are highlighted: best (orange) and second (blue).}
\label{tab:aircraft_results}
\begin{sc}
\resizebox{0.8\textwidth}{!}{%
\begin{tabular}{l|cccccc}
\toprule
\multicolumn{1}{c|}{\multirow{2}{*}{\textbf{Model}}} & \multicolumn{6}{c}{\textbf{AirCraft}} \\ 
\cmidrule(lr){2-7} 
\multicolumn{1}{c|}{} 
& $C_p$ $\downarrow$ & $\rho$ $\downarrow$ & $U$ $\downarrow$ & $V$ $\downarrow$ & $W$ $\downarrow$ & Pressure $\downarrow$ \\ 
\midrule

GraphSAGE \citep{hamilton2017inductive} & 0.1442 & 0.0803 & 0.0148 & 0.2154 & 0.2395 & 0.0923 \\
PointNet \citep{qi2017pointnet} & 0.1472 & 0.0783 & 0.0147 & 0.2104 & 0.2528 & 0.0937 \\
Graph U-Net \citep{gao2019graph} & 0.1475 & 0.0784 & 0.0152 & 0.2225 & 0.2477 & 0.0947 \\
MeshGraphNet \citep{pfaff2020learning} & 0.1372 & 0.0682 & 0.0164 & 0.1835 & 0.2175 & 0.0893 \\
GNO \citep{li2020neural} & 0.1556 & 0.0785 & 0.0167 & 0.2012 & 0.2318 & 0.1071 \\
Galerkin \citep{cao2021choose} & 0.0990 & 0.0429 & 0.0110 & 0.1304 & 0.1077 & 0.0480 \\
Geo-FNO \citep{li2023fourier} & 0.2586 & 0.1227 & 0.0300 & 0.3014 & 0.2757 & 0.1432 \\
GNOT \citep{hao2023gnot} & 0.0824 & 0.0407 & 0.0092 & 0.1141 & 0.1164 & \cellcolor{ranktwo}0.0436 \\
GINO \citep{li2023geometry} & 0.1513 & 0.0636 & 0.0160 & 0.1501 & 0.1431 & 0.0853 \\
Transolver \citep{wu2024transolver} & \cellcolor{ranktwo}0.0748 & \cellcolor{ranktwo}0.0379 & \cellcolor{ranktwo}0.0091 & 0.1107 & 0.1016 & 0.0439 \\
Transolver++ \citep{luotransolver++} & 0.0784 & 0.0390 & 0.0122 & 0.1318 & 0.1048 & 0.0445 \\
LaMO \citep{tiwarilatent} & 0.0786 & 0.0400 & 0.0099 & 0.1234 & 0.1094 & 0.0438 \\
LinearNO \citep{hu2025transolver} & 0.0800 & 0.0390 & 0.0095 & \cellcolor{ranktwo}0.1079 & \cellcolor{ranktwo}0.0972 & 0.0448 \\ 
\midrule

\textbf{Advectra (Ours)} 
& \cellcolor{rankone}\textbf{0.0660} 
& \cellcolor{rankone}\textbf{0.0331} 
& \cellcolor{rankone}\textbf{0.0057} 
& \cellcolor{rankone}\textbf{0.0862} 
& \cellcolor{rankone}\textbf{0.0863} 
& \cellcolor{rankone}\textbf{0.0372} \\ 

\bottomrule
\end{tabular}%
}
\end{sc}
\end{table}

\section{Additional Analysis of Transport Behavior}
\label{app:transport_analysis}

To further characterize the transport properties of Advectra, we analyze the learned operator through moment-based diagnostics that isolate conservation, advection, and structural fidelity.

\subsection{Moment-Based Transport Diagnostics}

Let $c(x)$ denote a scalar tracer field over spatial domain $\Omega$. We evaluate transport behavior through the following quantities:

\paragraph{Mass (0\textsuperscript{th} moment).}
\[
M = \int_{\Omega} c(x)\,dx, \quad
\mathcal{E}_{\text{mass}} = \frac{|M^{\text{pred}} - M^{\text{gt}}|}{M^{\text{gt}}}.
\]
This measures conservation of the transported quantity.

\paragraph{Center of Mass (1\textsuperscript{st} moment).}
\[
\mu = \frac{\int_{\Omega} x\,c(x)\,dx}{\int_{\Omega} c(x)\,dx}, \quad
\mathcal{E}_{\text{COM}} = \|\mu^{\text{pred}} - \mu^{\text{gt}}\|_2.
\]
This captures the integrated displacement induced by advection.

\paragraph{Pointwise Error.}
\[
\mathcal{E}_{\text{raw}} = \|c^{\text{pred}} - c^{\text{gt}}\|_1,
\]
which measures structural fidelity of the transported field.

\subsection{COM-Aligned Error}

To separate transport from deformation, we define a COM-aligned error by translating the prediction:
\[
\Delta = \mu^{\text{gt}} - \mu^{\text{pred}}, \quad
\mathcal{E}_{\text{aligned}} = \|c^{\text{pred}}(x - \Delta) - c^{\text{gt}}(x)\|_1.
\]

We report the ratio:
\[
R = \frac{\mathbb{E}[\mathcal{E}_{\text{aligned}}]}{\mathbb{E}[\mathcal{E}_{\text{raw}}]}.
\]

\paragraph{Interpretation.}
Since translation is an isometry under the $L_1$ norm, this ratio isolates structural effects:
\begin{itemize}
\item $R < 1$: a global COM shift reduces error, suggesting bulk misalignment contributes to the error,
\item $R \approx 1$: global translation provides little correction, suggesting residual error is not dominated by a missing bulk shift,
\item $R > 1$: global alignment worsens error, suggesting spatially nonuniform deformation, diffusion, or local misalignment.
\end{itemize}

\subsection{Results}

\begin{table}[h]
\centering
\caption{Transport diagnostics on NS-Tracer-PwC.}
\label{tab:COM}
\begin{tabular}{l|cccc}
\toprule
Model & COM $\downarrow$ & Mass $\downarrow$ & Raw MAE $\downarrow$ & Ratio $R$ $\downarrow$ \\
\midrule
Advectra & \textbf{0.0043} & \textbf{0.67\%} & \textbf{0.0285} & \textbf{1.0099} \\
Advectra (No Transport) & 0.0071 & 0.93\% & 0.0419 & 1.0656 \\
LaMO & 0.0072 & 1.10\% & 0.0397 & 1.0973 \\
\bottomrule
\end{tabular}
\end{table}

Advectra achieves a $\sim 40\%$ relative reduction in COM error compared to both LaMO and the no-transport variant, indicating more accurate first-moment alignment of the predicted tracer. It also achieves lower mass error and lower raw MAE, showing better preservation of tracer mass and field-level tracer structure.


Advectra exhibits a ratio \(R\approx 1.01\), meaning that a global COM correction changes its error by only \(\sim 1\%\). Thus, its residual error is not dominated by a missing bulk translation. In contrast, the no-transport variant has \(R=1.0656\), meaning that its COM-aligned error is \(6.5\%\) higher than its own raw error; LaMO has \(R=1.0973\), corresponding to a \(9.7\%\) increase relative to its own raw error. In absolute terms, the COM-aligned error increases from approximately \(0.0288\) for Advectra to \(0.0446\) for the no-transport variant and \(0.0436\) for LaMO. These increases under global alignment indicate that the discrepancies are not simple translations of the ground truth, but are more consistent with spatially nonuniform distortions such as excessive diffusion, incorrect stretching, or loss of filamentary detail.

Together, these diagnostics provide direct evidence that explicit latent transport improves both advective alignment and structure preservation in transport-dominated dynamics: Advectra better matches global displacement, total tracer mass, and local tracer profiles without relying on a post-hoc global shift.

\subsection{Qualitative Evidence}
Figure~\ref{fig:com_aligned_errors} further supports the diagnostic results. The COM-aligned error comparison shows that Advectra remains nearly unchanged after alignment, while both baselines exhibit noticeable degradation. The 1D slice visualizations in Figure~\ref{fig:1d-slices} show that Advectra better preserves the peak amplitude, spatial alignment, and sharpness of tracer profiles, whereas the baselines exhibit broader profiles, shifted centers, and stronger structural distortion.

\begin{figure}[htbp]
    \centering
    \begin{subfigure}[b]{\textwidth}
        \centering
        \includegraphics[width=0.8\textwidth, height=0.25\textheight, keepaspectratio]{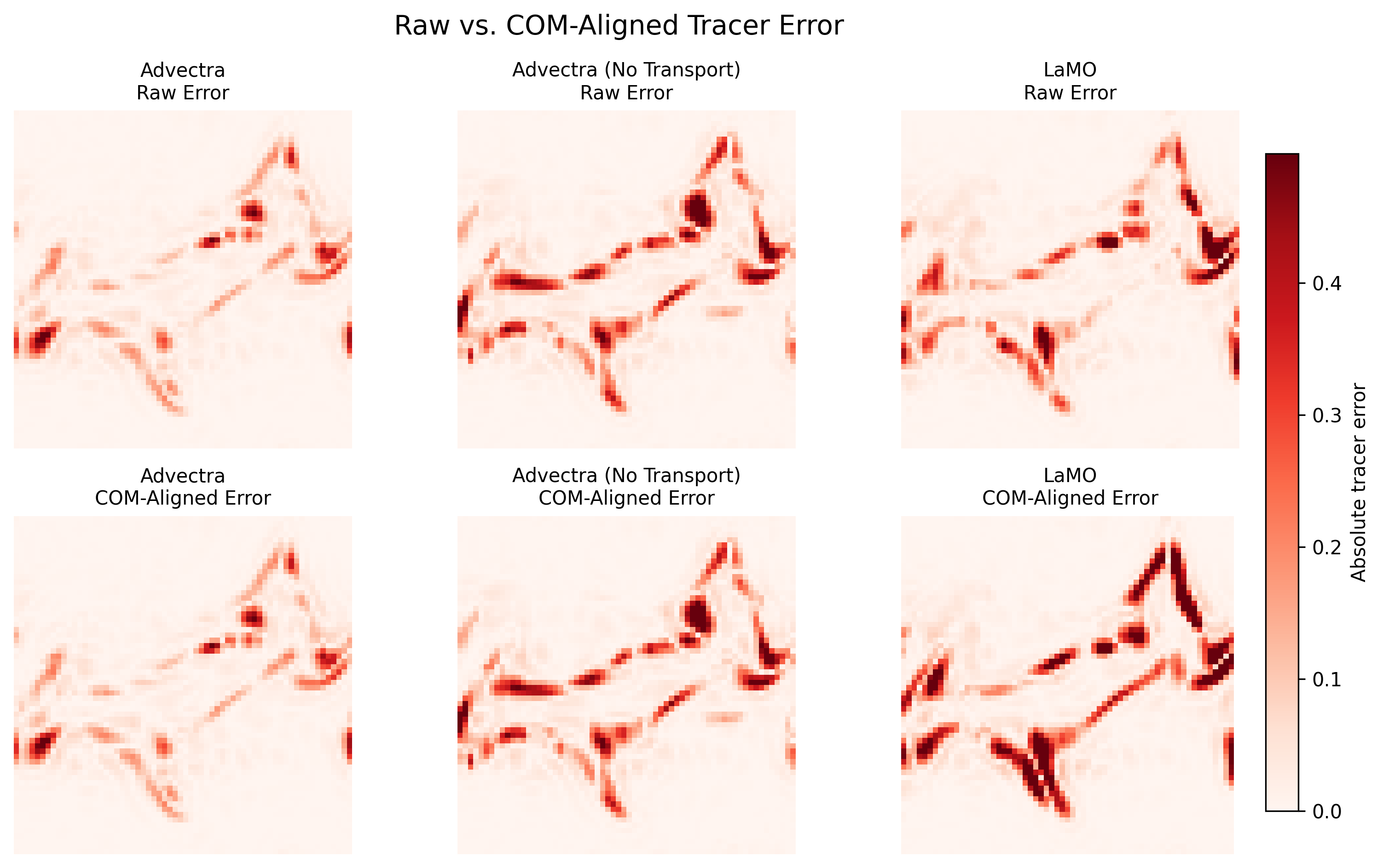}
        \caption{Sample Trajectory 1}
        \label{fig:aligned_error_0}
    \end{subfigure}
    
    \vspace{1em} 
    
    \begin{subfigure}[b]{\textwidth}
        \centering
        \includegraphics[width=0.8\textwidth, height=0.25\textheight, keepaspectratio]{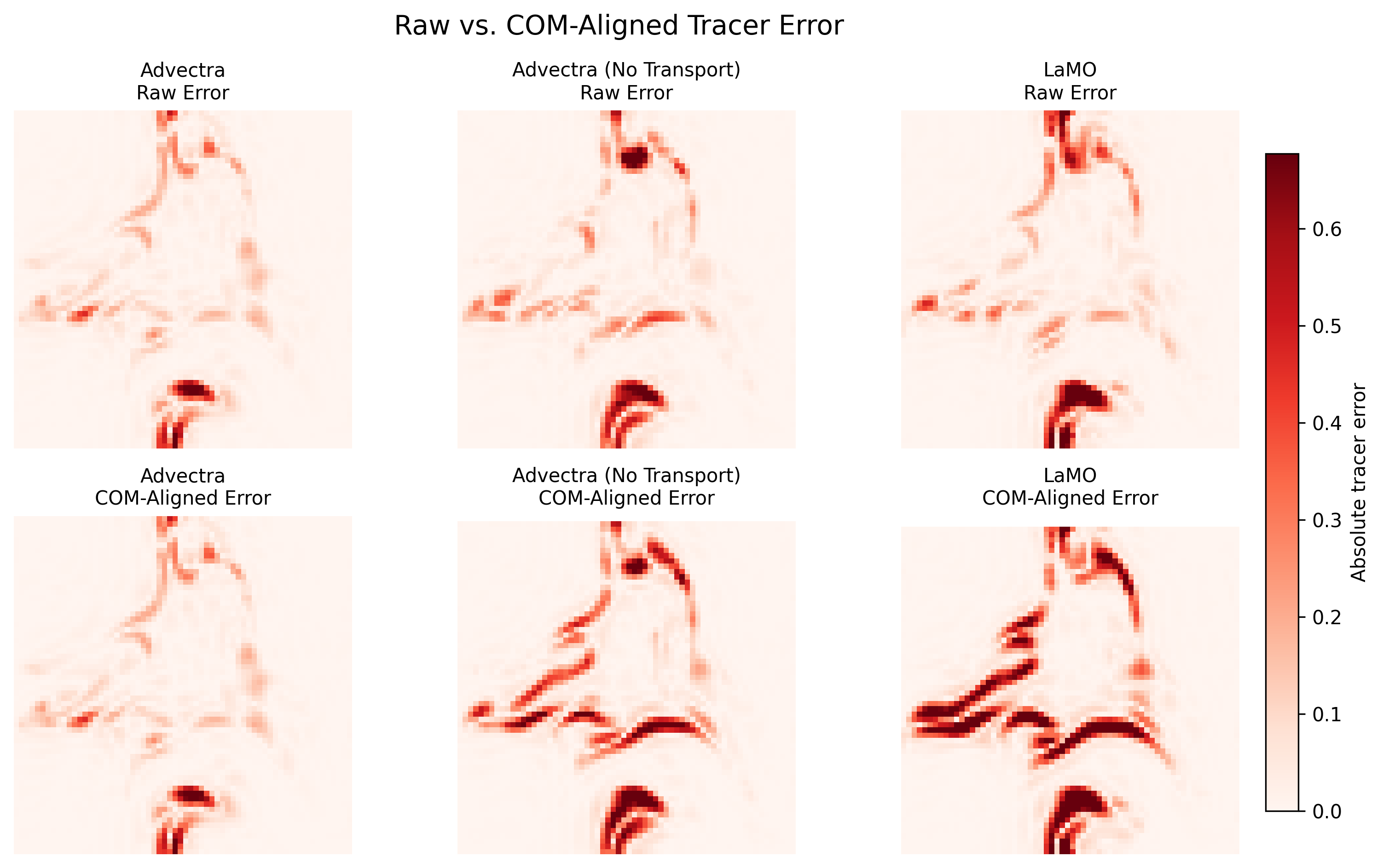}
        \caption{Sample Trajectory 2}
        \label{fig:aligned_error_1}
    \end{subfigure}
    
    \vspace{1em} 
    
    \begin{subfigure}[b]{\textwidth}
        \centering
        \includegraphics[width=0.8\textwidth, height=0.25\textheight, keepaspectratio]{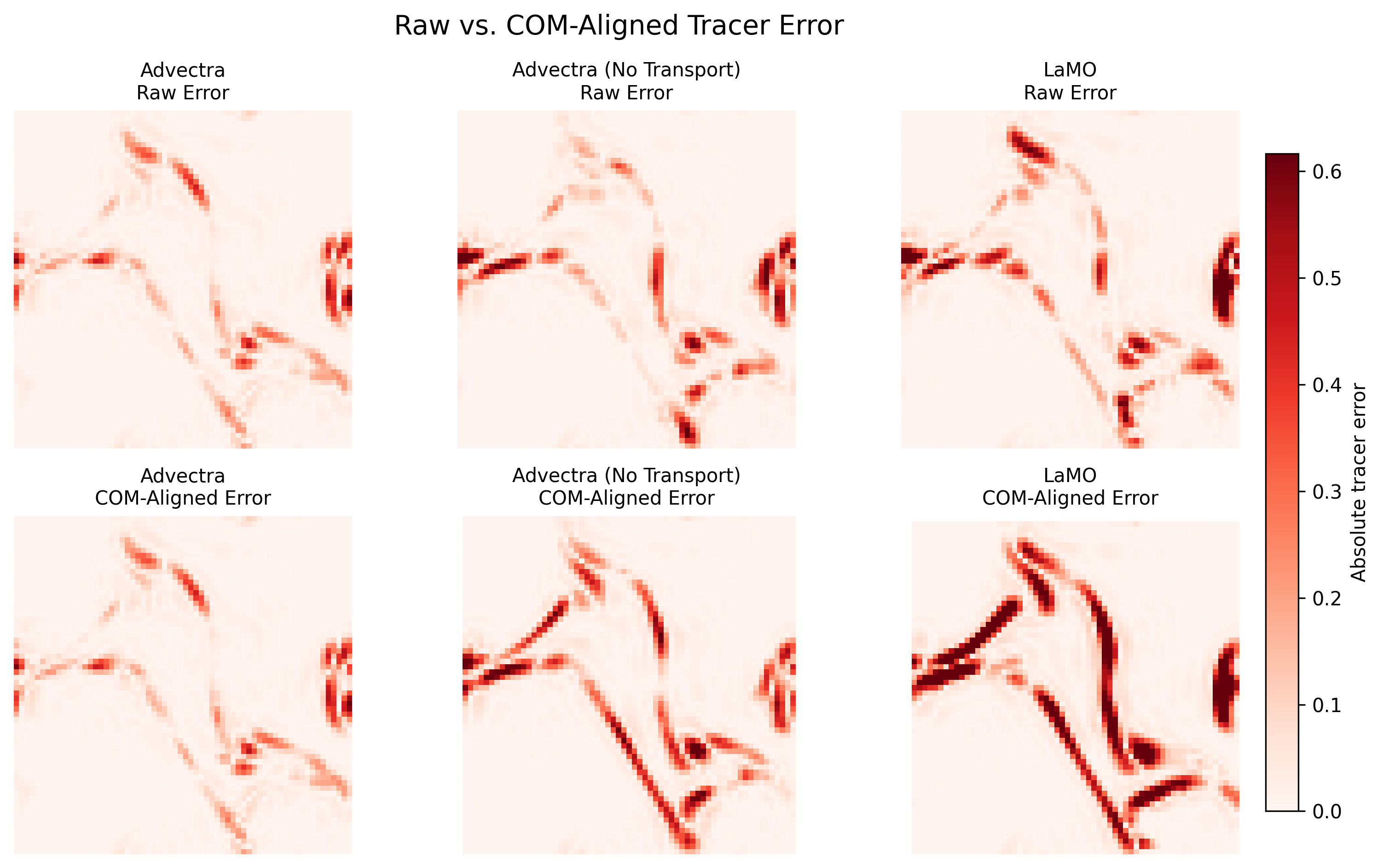}
        \caption{Sample Trajectory 3}
        \label{fig:aligned_error_2}
    \end{subfigure}
    
    \caption{Comparison of Raw Error versus Center-of-Mass (COM) Aligned Error across three randomly selected test trajectories. By translating the predicted tracer distribution by the displacement vector $d=\mu_{\mathrm{gt}}-\mu_{\mathrm{pred}}$ before computing pointwise error, this diagnostic separates bulk COM misalignment from residual structural mismatch. Higher COM-aligned error indicates deformation, diffusion, or local misalignment that cannot be corrected by a global translation, while lower COM-aligned error indicates better preservation of tracer structure under this diagnostic.}
    \label{fig:com_aligned_errors}
\end{figure}

\begin{figure}[htbp]
    \centering
    \begin{subfigure}[b]{\textwidth}
        \centering
        \includegraphics[width=0.8\textwidth, height=0.25\textheight, keepaspectratio]{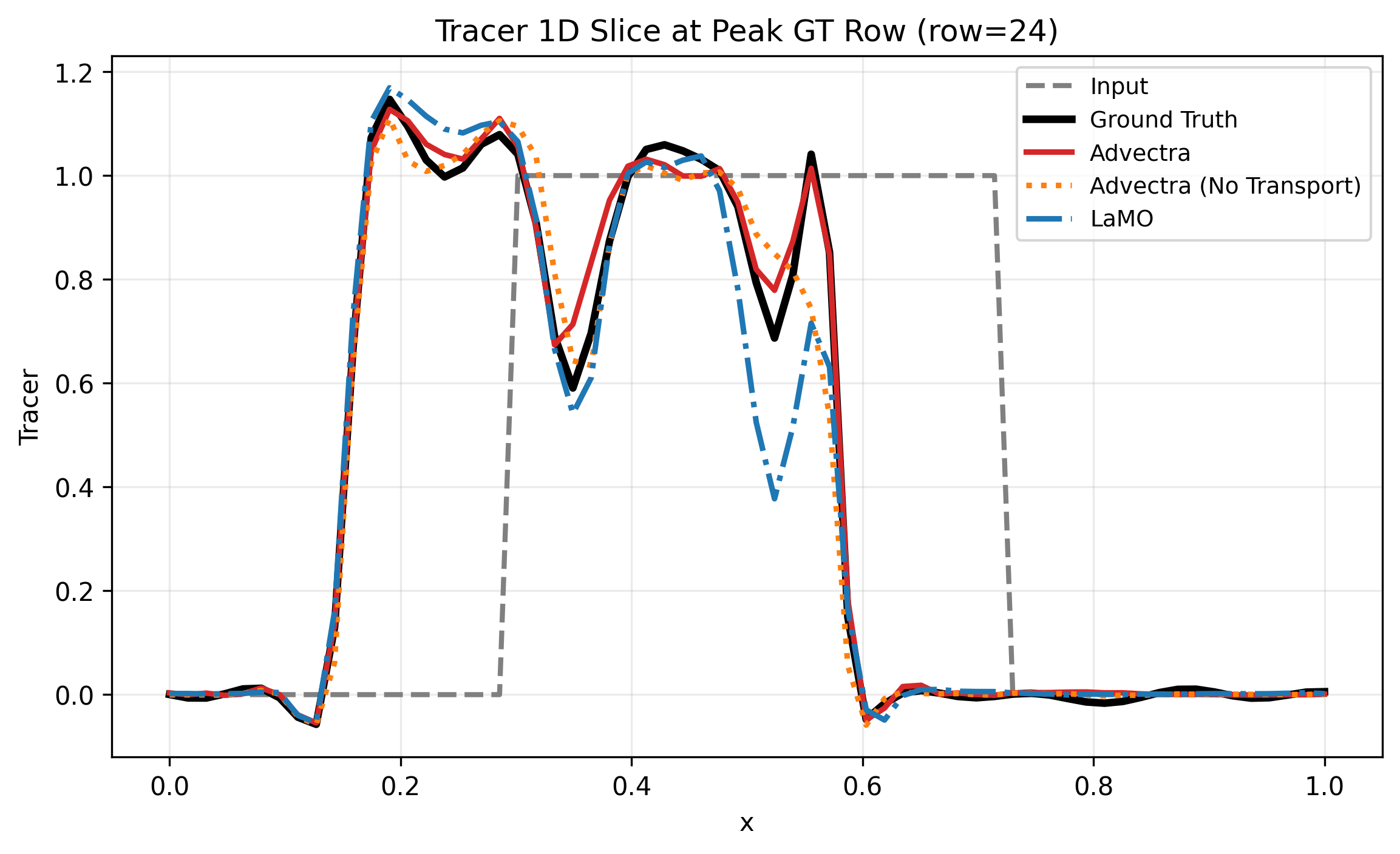}
        \caption{Sample Trajectory 1}
        \label{fig:slice_1d_0}
    \end{subfigure}
    
    \vspace{1em} 
    
    \begin{subfigure}[b]{\textwidth}
        \centering
        \includegraphics[width=0.8\textwidth, height=0.25\textheight, keepaspectratio]{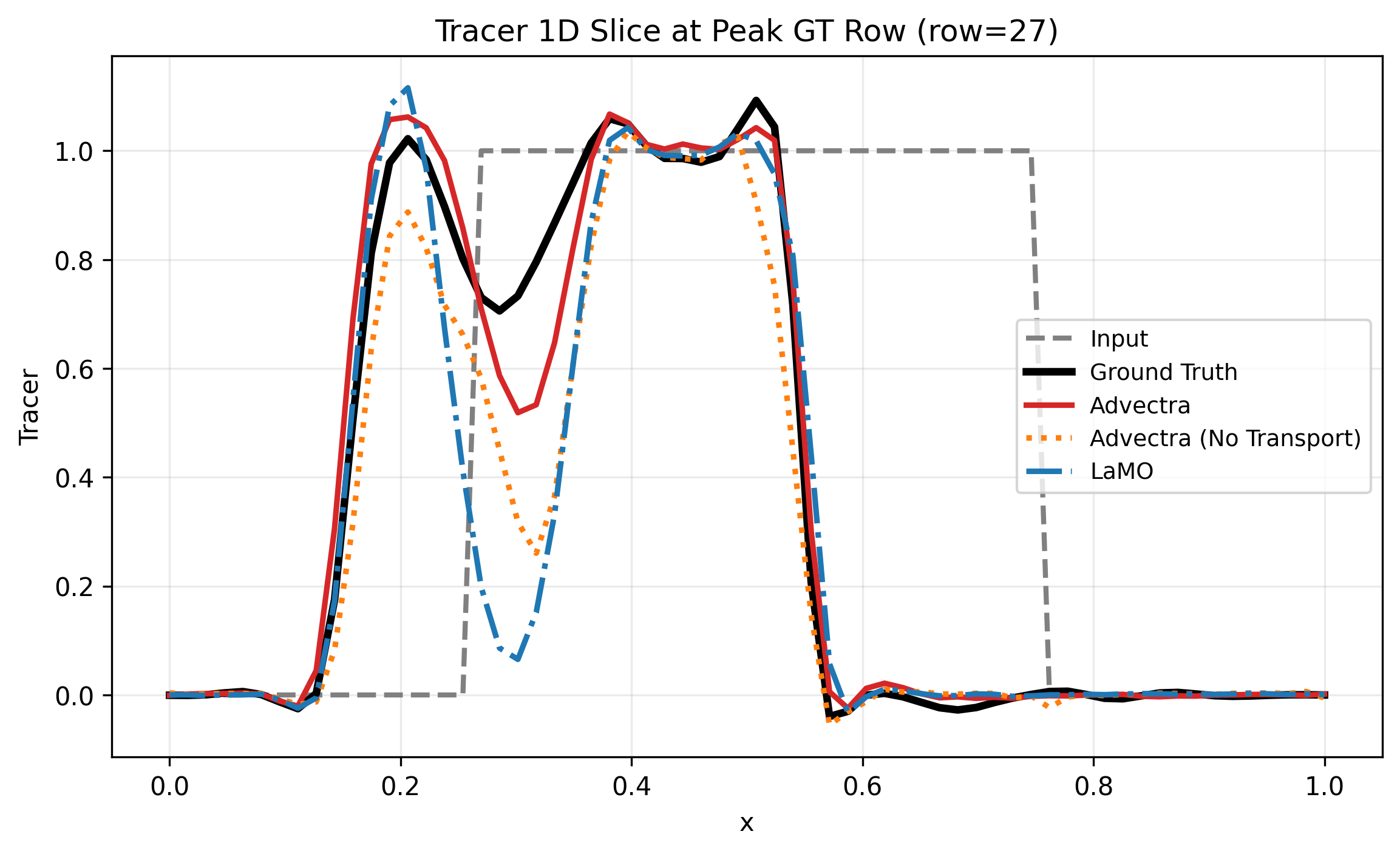}
        \caption{Sample Trajectory 2}
        \label{fig:slice_1d_1}
    \end{subfigure}
    
    \vspace{1em} 
    
    \begin{subfigure}[b]{\textwidth}
        \centering
        \includegraphics[width=0.8\textwidth, height=0.25\textheight, keepaspectratio]{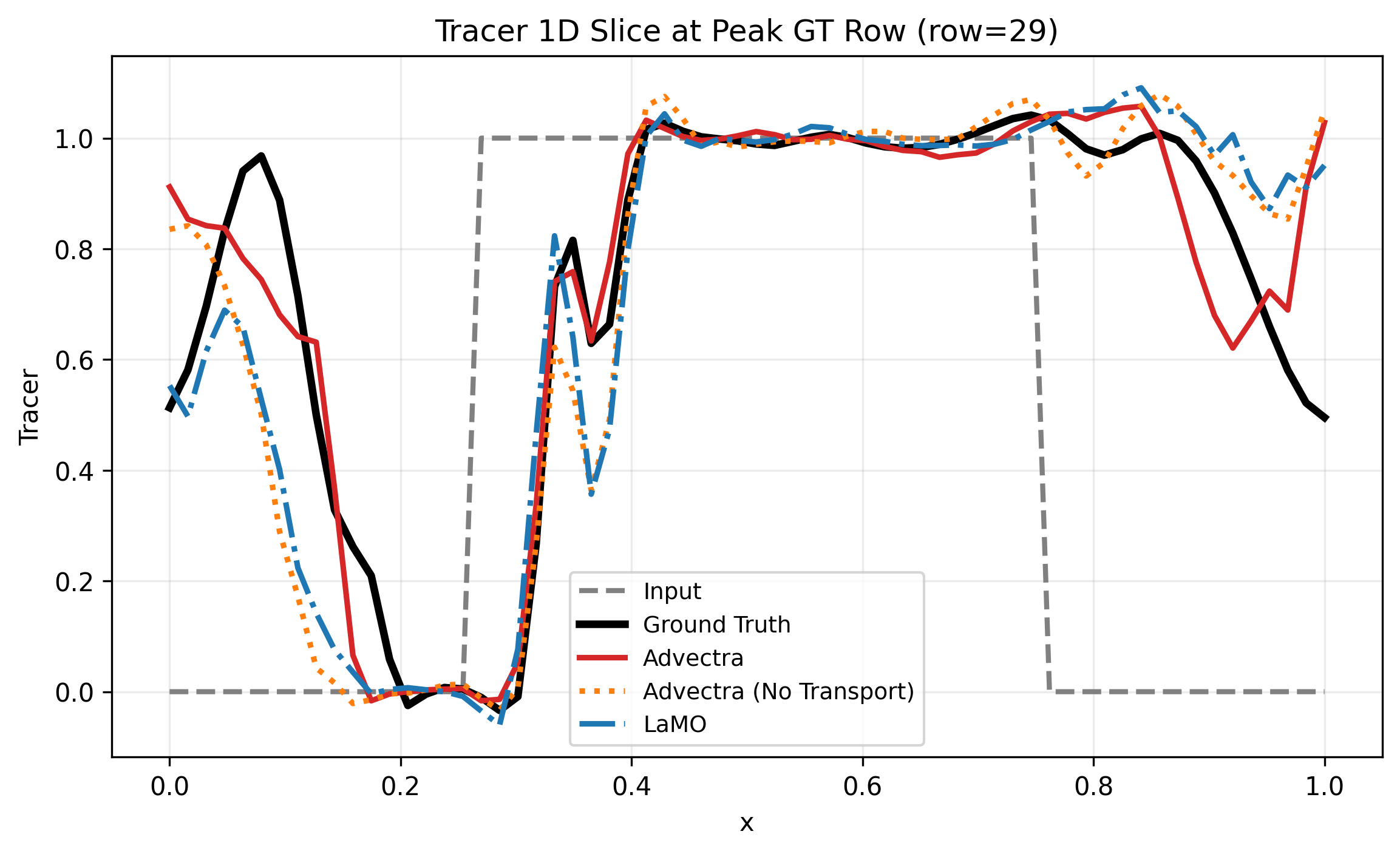}
        \caption{Sample Trajectory 3}
        \label{fig:slice_1d_2}
    \end{subfigure}
    
    \caption{1D horizontal cross-sections of the tracer concentration taken at the spatial peak for three randomly selected test trajectories. Advectra produces profiles that more closely match the ground truth in peak amplitude and spatial alignment. In contrast, the baseline models exhibit lower peak amplitudes, broader spatial spreads, and shifted centers, indicating stronger diffusion and advection lag.}
    \label{fig:1d-slices}
\end{figure}

\begin{figure}[htbp]
    \centering
    \begin{subfigure}[b]{\textwidth}
        \centering
        \includegraphics[width=0.8\textwidth, height=0.25\textheight, keepaspectratio]{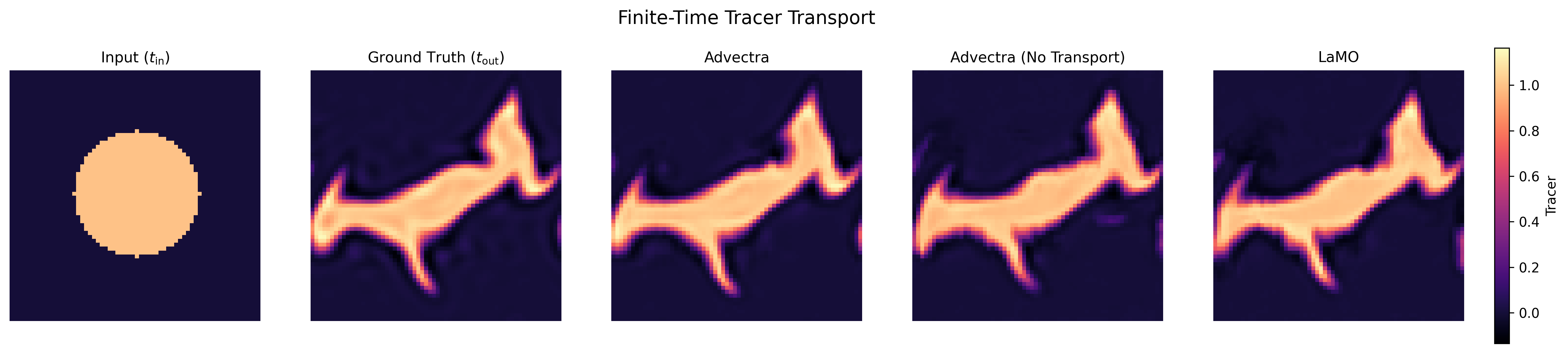}
        \caption{Sample Trajectory 1}
        \label{fig:transport_panel_0}
    \end{subfigure}
    
    \vspace{1em} 
    
    \begin{subfigure}[b]{\textwidth}
        \centering
        \includegraphics[width=0.8\textwidth, height=0.25\textheight, keepaspectratio]{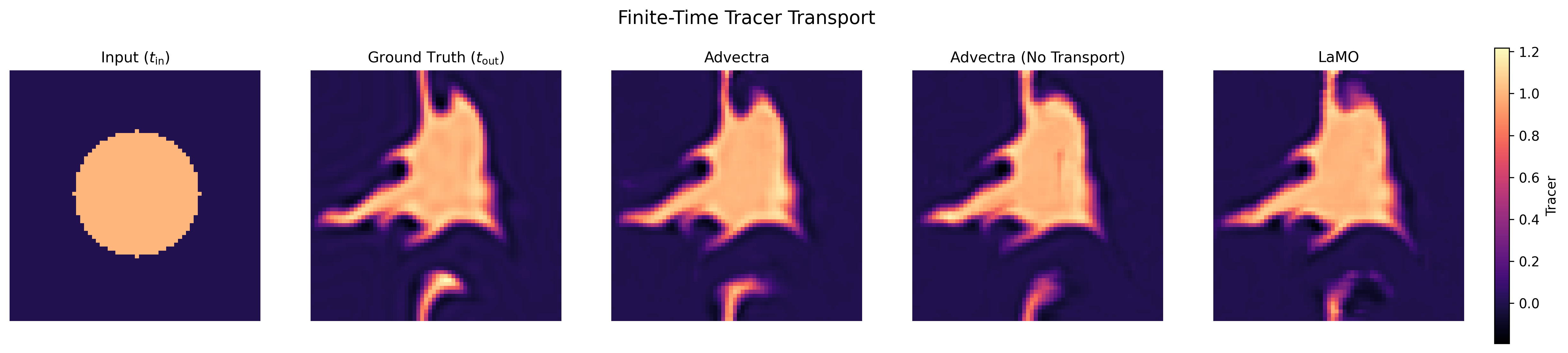}
        \caption{Sample Trajectory 2}
        \label{fig:transport_panel_1}
    \end{subfigure}
    
    \vspace{1em} 
    
    \begin{subfigure}[b]{\textwidth}
        \centering
        \includegraphics[width=0.8\textwidth, height=0.25\textheight, keepaspectratio]{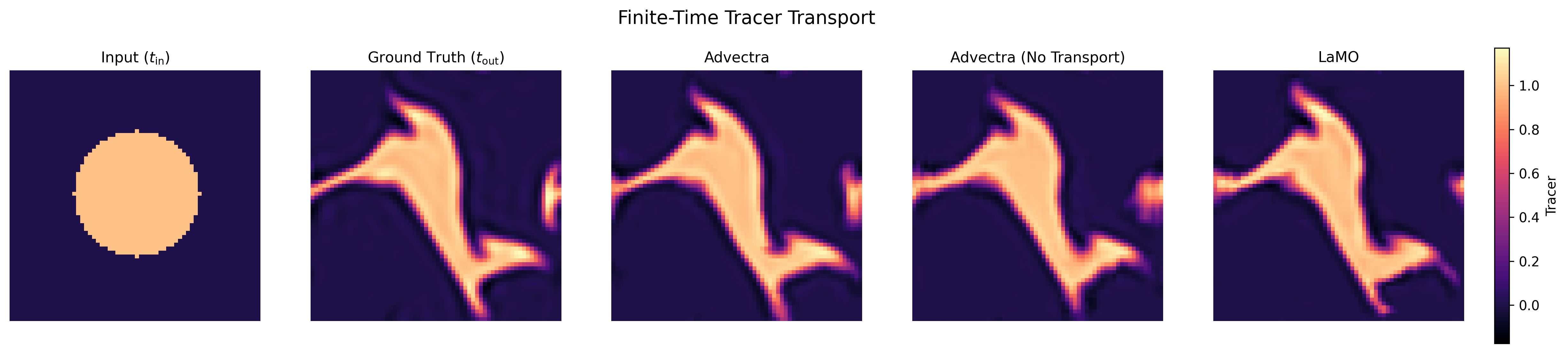}
        \caption{Sample Trajectory 3}
        \label{fig:transport_panel_2}
    \end{subfigure}
    
    \caption{Qualitative comparison of predicted tracer transport across three randomly selected test trajectories. Each panel contrasts the Ground Truth state at $t_{out}$ against the predictions generated by the proposed Advectra architecture, the ablated Advectra (No Transport) variant, and the LaMO baseline. The visualizations illustrate the differing spatial distributions and structural deformations produced by each operator learning formulation.}
    \label{fig:transport-panels}
\end{figure}

\begin{figure}[htbp]
    \centering
    \includegraphics[width=0.75\textwidth, height=0.35\textheight, keepaspectratio]{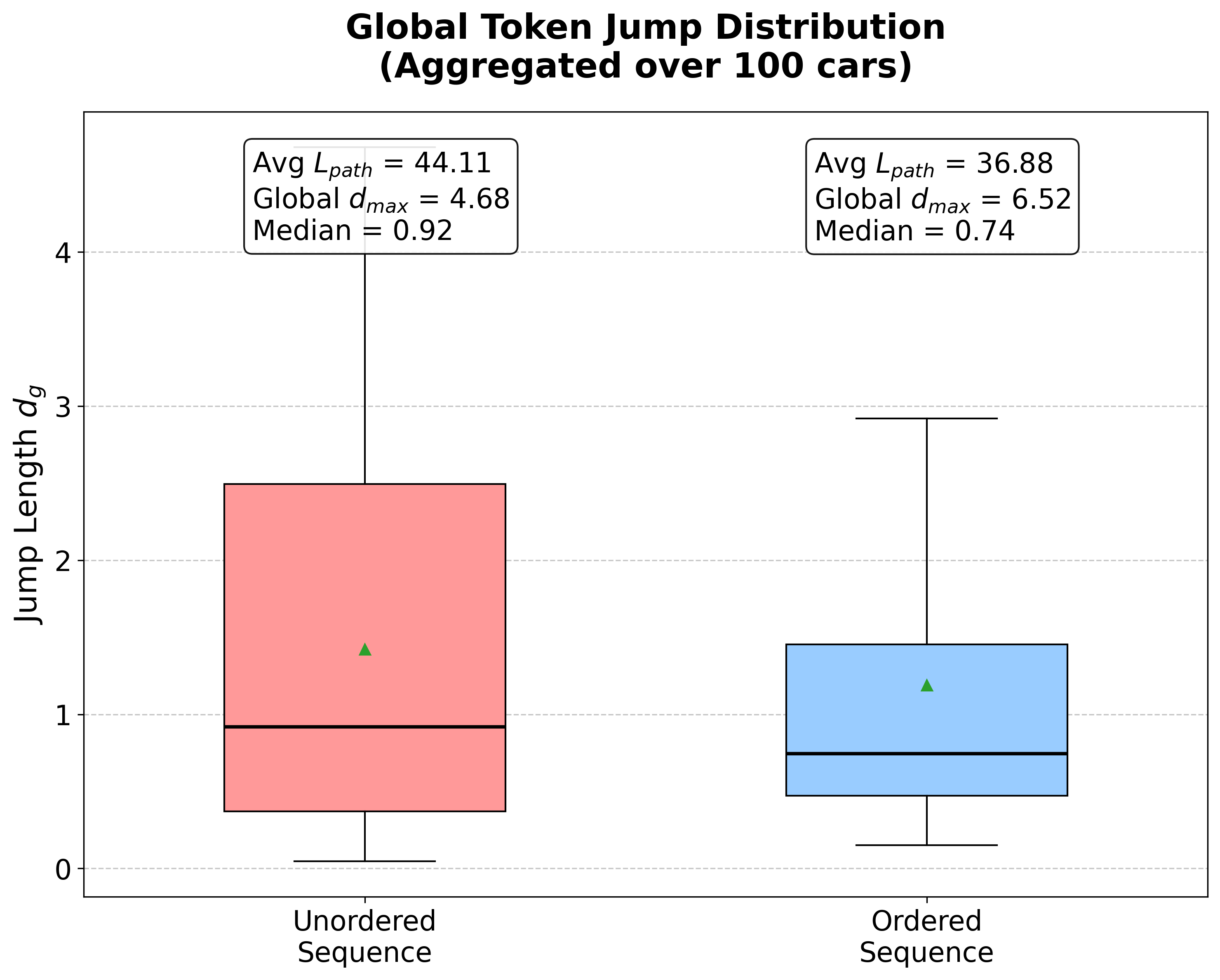}
    \caption{Distribution of consecutive centroid distances for randomly selected test samples. Geometry-aware ordering reduces both the average and median jump distance, indicating improved locality. At the same time, it preserves a small number of larger jumps, corresponding to structured long-range transitions necessary to capture global geometry.}
    \label{fig:jump_distribution}
\end{figure}

\section{Geometric Analysis of Latent Ordering}
\label{sec:ordering_analysis}

To understand the role of geometry-aware ordering in Advectra, we analyze how the latent sequence aligns with the underlying physical geometry. In particular, we examine whether sorting encourages spatial coherence, which is important for effective sequential state-space modeling.

\subsection{Latent Paths in Physical Space}

We first visualize how the latent sequence traverses the physical domain. Each latent slice is mapped to its corresponding centroid in physical space, and the ordering induces a path connecting these centroids.

\begin{figure}[htbp]
    \centering
    \begin{subfigure}[b]{\textwidth}
        \centering
        \includegraphics[width=0.85\textwidth, height=0.23\textheight, keepaspectratio]{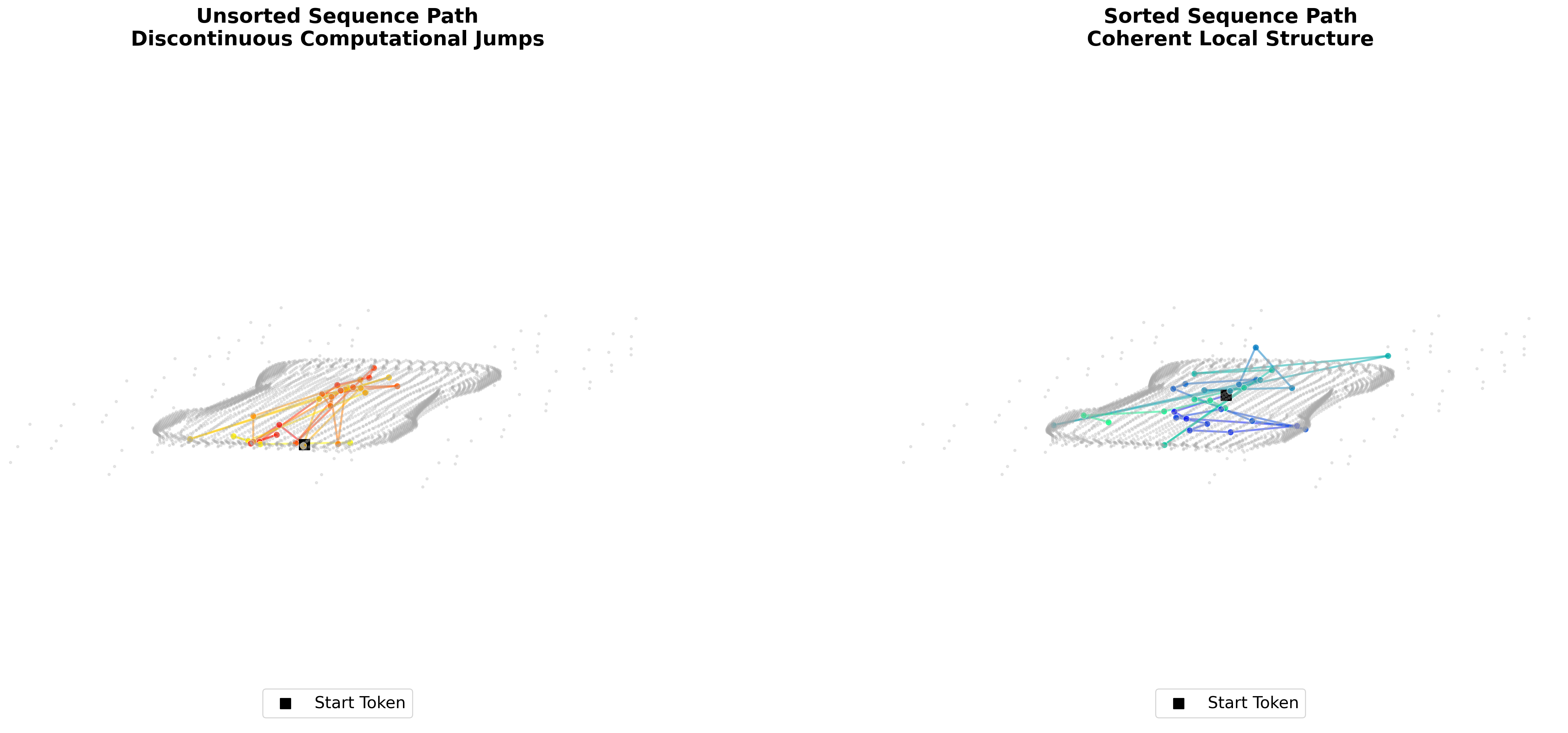}
        \caption{Randomly selected test geometry 1}
        \label{fig:centroid_paths_case_1}
    \end{subfigure}
    
    \vspace{0.6em}
    
    \begin{subfigure}[b]{\textwidth}
        \centering
        \includegraphics[width=0.85\textwidth, height=0.23\textheight, keepaspectratio]{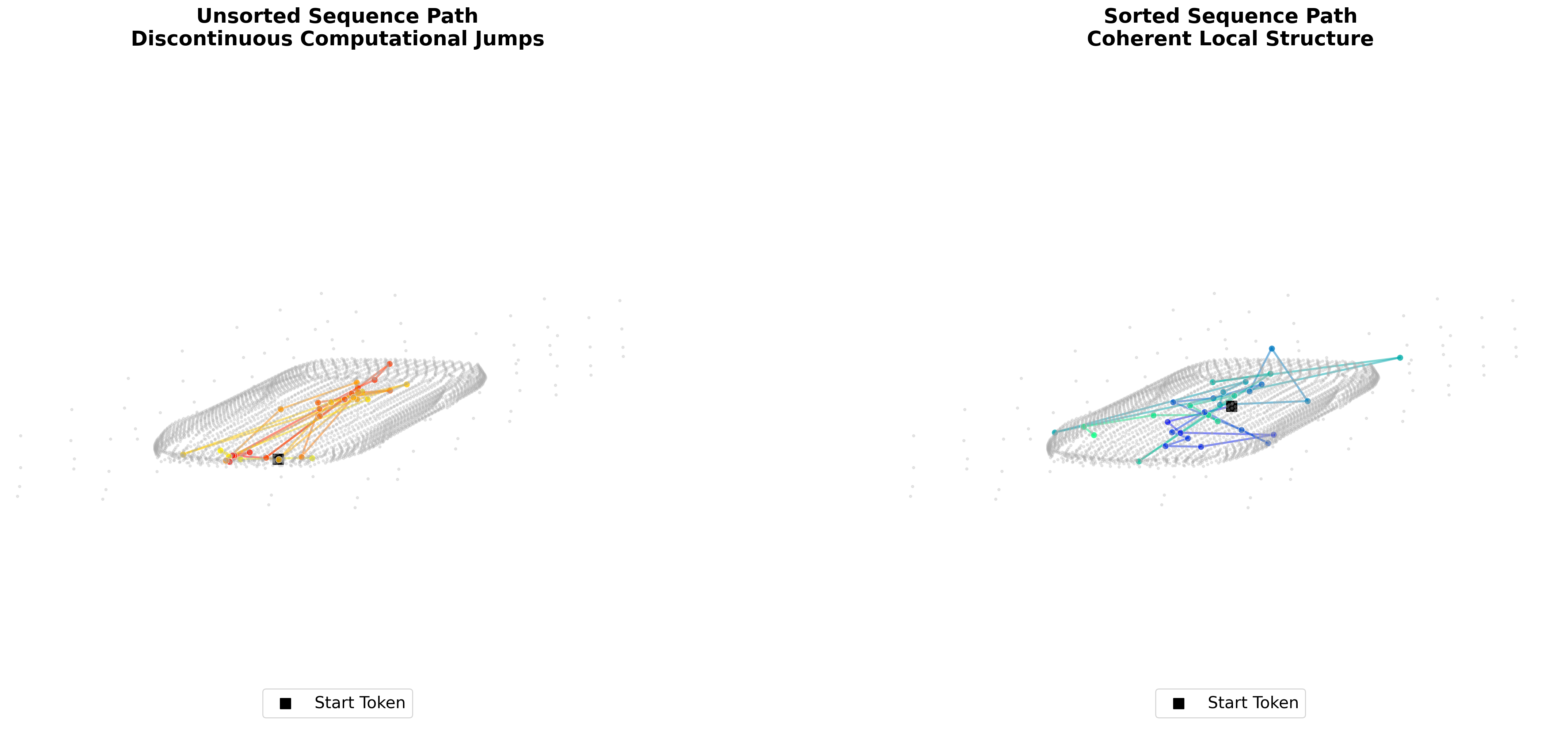}
        \caption{Randomly selected test geometry 2}
        \label{fig:centroid_paths_case_2}
    \end{subfigure}
    
    \vspace{0.6em}
    
    \begin{subfigure}[b]{\textwidth}
        \centering
        \includegraphics[width=0.85\textwidth, height=0.23\textheight, keepaspectratio]{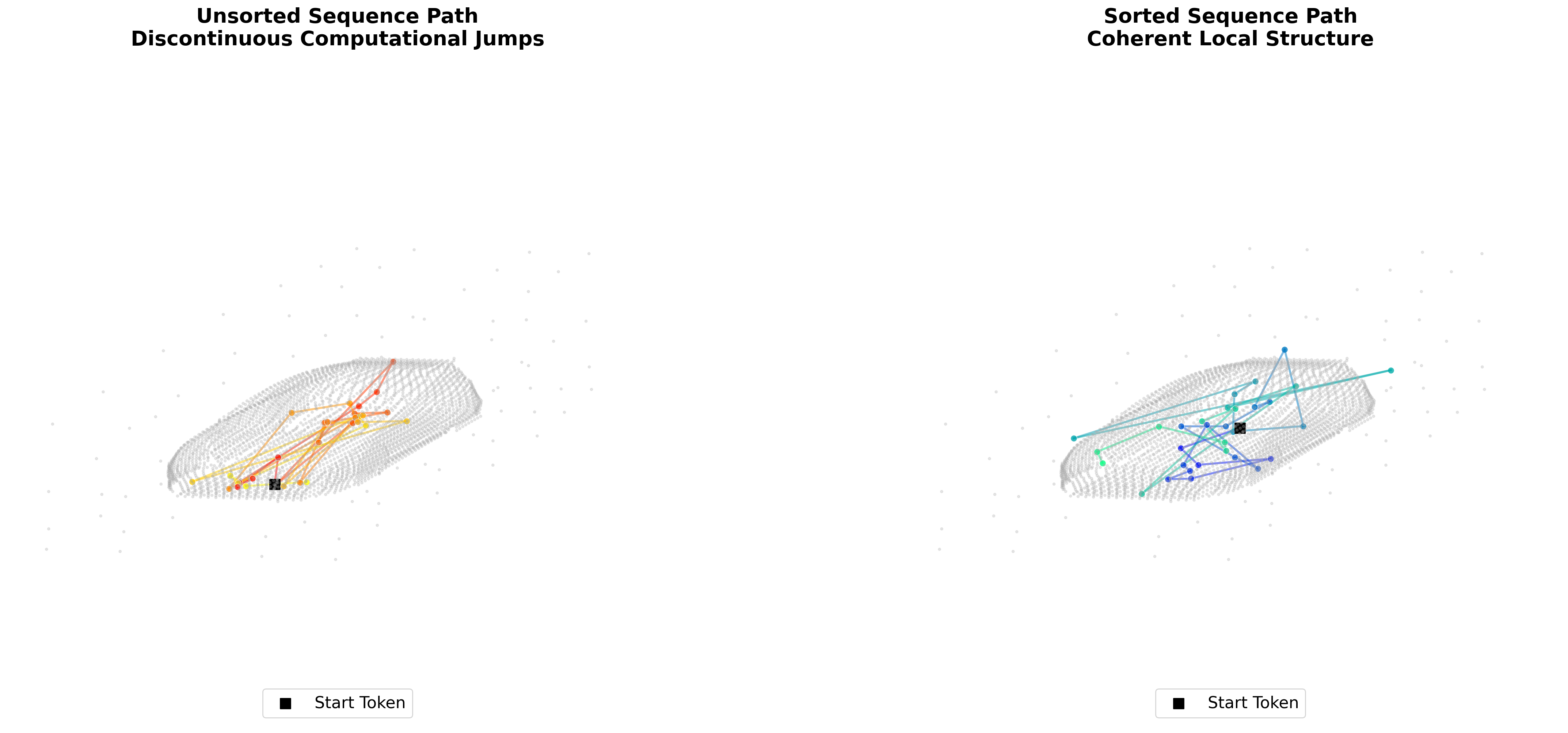}
        \caption{Randomly selected test geometry 3}
        \label{fig:centroid_paths_case_3}
    \end{subfigure}
    
    \caption{Latent sequence paths mapped to physical space for three randomly selected ShapeNet-Car test geometries. Without ordering, the sequence exhibits frequent long-range jumps between spatially discontinuous regions (e.g., front and rear of the geometry). With geometry-aware ordering, the paths become smoother and spatially coherent, traversing the domain through predominantly local transitions with only occasional long-range connections.}
    \label{fig:centroid_paths}
\end{figure}

As shown in Figure~\ref{fig:centroid_paths}, the unordered latent sequence exhibits frequent long-range jumps between spatially disconnected regions. This indicates that consecutive tokens correspond to unrelated physical locations, leading to incoherent interactions when processed sequentially.

In contrast, geometry-aware ordering produces geometrically consistent paths that traverse the domain smoothly. Most transitions occur locally between neighboring regions, while a small number of long-range jumps are retained to connect distant but physically meaningful regions.

\subsection{Distribution of Jump Distances}

To quantify this behavior, we measure the Euclidean distance between consecutive latent centroids along the sequence.

Figure~\ref{fig:jump_distribution} shows that geometry-aware ordering reduces the average path length from $44.11$ to $36.88$ (approximately $16.4\%$ reduction) and the median jump distance from $0.92$ to $0.74$, indicating that most adjacent transitions become more localized.

At the same time, the maximum jump distance increases from $4.68$ to $6.52$. This shows that a small number of long-range transitions remain after ordering. Rather than eliminating global interactions, the ordering concentrates them into sparse, structured connections while improving locality for most adjacent token pairs.

\paragraph{Interpretation.}
These results indicate that geometry-aware ordering reorganizes the latent sequence into predominantly local interactions, while retaining a small number of structured global transitions. This balance is useful for SSM-based operators because it encourages sequential updates to follow meaningful spatial relationships.

\subsection{Aggregation Field Visualization}

Finally, we visualize the effective aggregation field using quiver plots.

\begin{figure}[htbp]
    \centering
    \begin{subfigure}[b]{\textwidth}
        \centering
        \includegraphics[width=0.85\textwidth, height=0.23\textheight, keepaspectratio]{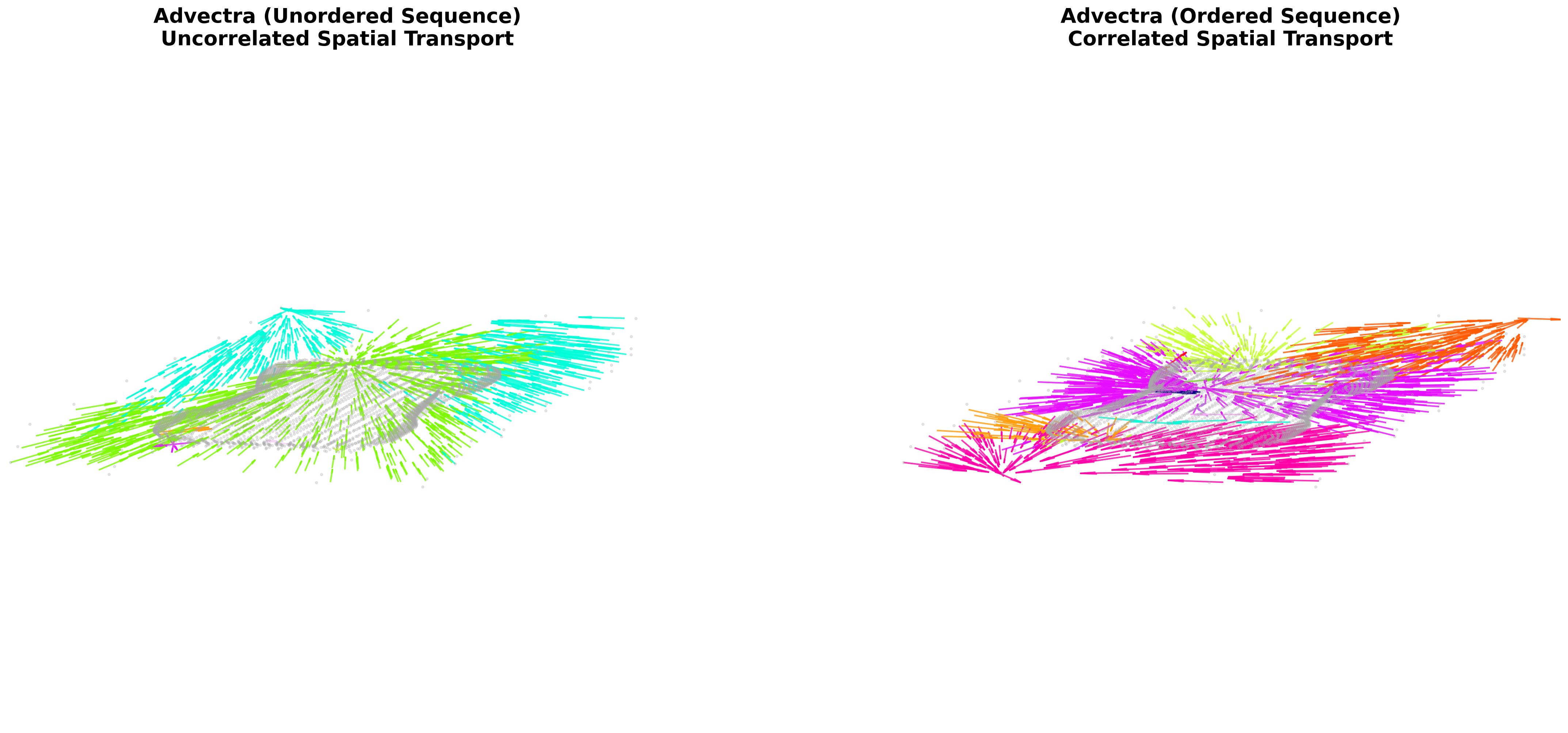}
        \caption{Randomly selected test geometry 1}
        \label{fig:quiver_case_1}
    \end{subfigure}
    
    \vspace{0.6em}
    
    \begin{subfigure}[b]{\textwidth}
        \centering
        \includegraphics[width=0.85\textwidth, height=0.23\textheight, keepaspectratio]{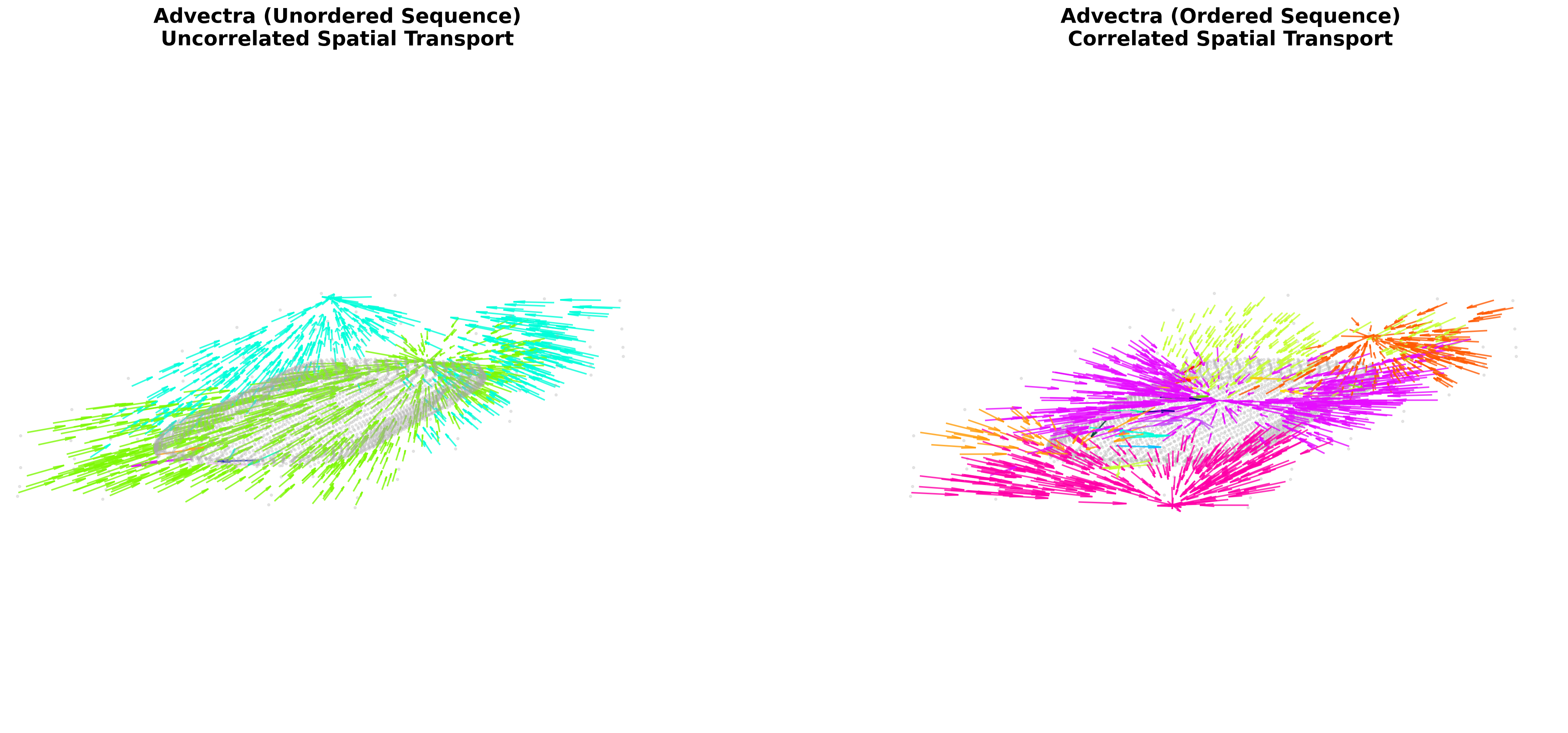}
        \caption{Randomly selected test geometry 2}
        \label{fig:quiver_case_2}
    \end{subfigure}
    
    \vspace{0.6em}
    
    \begin{subfigure}[b]{\textwidth}
        \centering
        \includegraphics[width=0.85\textwidth, height=0.23\textheight, keepaspectratio]{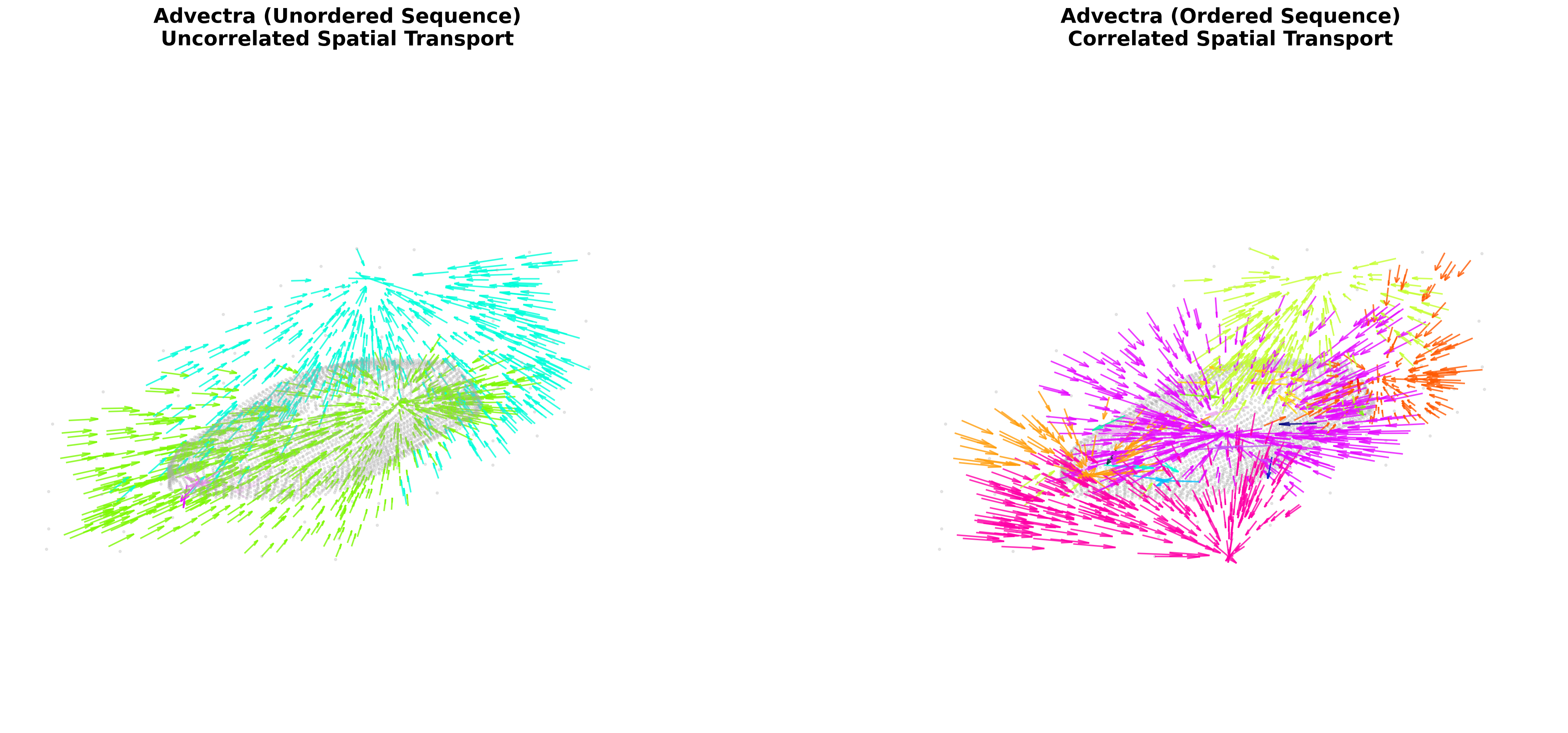}
        \caption{Randomly selected test geometry 3}
        \label{fig:quiver_case_3}
    \end{subfigure}
    
    \caption{Aggregation field visualization for three randomly selected ShapeNet-Car test geometries. Without ordering, the field exhibits long-range, crossing interactions spanning large portions of the domain, indicating spatially incoherent aggregation. With geometry-aware ordering, the field decomposes into localized, spatially coherent clusters with bounded extent, reflecting predominantly local interactions.}
    \label{fig:quiver}
\end{figure}

Without ordering, the aggregation field exhibits degenerate behavior, where a small number of dominant directions span large regions and heavily cross each other. This indicates that features are aggregated from distant and unrelated locations.

With geometry-aware ordering, the field becomes structured and localized. The vectors form coherent clusters with bounded spatial extent, indicating that information is primarily aggregated from nearby regions.

\paragraph{Summary.}
Together, these results show that geometry-aware ordering aligns the latent sequence more closely with the underlying physical geometry. This alignment helps sequential updates correspond to meaningful spatial interactions rather than arbitrary mixing across distant regions.

\clearpage
\section{Effect of Ordering Frequency}
\label{sec:sorting_frequency}

We investigate the sensitivity of Advectra to the frequency of geometry-aware ordering on the \textbf{ShapeNet-Car} dataset, a challenging 3D aerodynamic benchmark where preserving spatial structure under latent compression is particularly critical. While ordering enforces geometric coherence in the latent sequence, it introduces additional computational overhead. A key question is therefore whether frequent reordering is necessary, or whether occasional updates suffice.

\begin{table}[H]
\centering
\caption{Effect of ordering frequency on the \textbf{ShapeNet-Car} dataset. We report relative $L_2$ errors on volume and surface fields. Geometry-aware ordering significantly improves performance compared to no ordering, while results remain stable across a wide range of sorting intervals (i.e., refresh intervals of the nearest-neighbor latent ordering).}
\label{tab:sorting_frequency}
\begin{sc}
\resizebox{0.5\textwidth}{!}{%
\begin{tabular}{c|cc}
\toprule
\textbf{Sorting Interval} & \multicolumn{2}{c}{\textbf{ShapeNet-Car}} \\ \cmidrule(lr){2-3}
 & Volume $\downarrow$ & Surf $\downarrow$ \\ \midrule
25  & 0.0158 & 0.0736 \\
50  & 0.0150 & 0.0719\\
100 & \textbf{0.0149} & \textbf{0.0716}  \\
200 & 0.0153 & \textbf{0.0716}  \\
400 & 0.0156 & 0.0718  \\ \midrule
No Sorting & 0.0169 & 0.0766  \\
\bottomrule
\end{tabular}
}
\end{sc}
\end{table}

Table~\ref{tab:sorting_frequency} reports performance across a range of sorting intervals (i.e., the refresh interval of the geometry-aware nearest-neighbor ordering), together with a no-ordering baseline. Enabling ordering leads to a clear improvement over the unordered variant, reducing surface error from 0.0766 to approximately 0.0716 and volume error from 0.0169 to 0.0149. This shows that geometrically coherent latent ordering is important for accurate modeling on this complex 3D geometry benchmark.

Once ordering is introduced, the model is robust to the exact refresh frequency. Varying the sorting interval from 25 to 400 results in only minor performance fluctuations. The best results are obtained around an interval of 100, but all ordered configurations achieve comparable accuracy.

\paragraph{Interpretation.}
These results indicate that ordering primarily acts as a periodic geometric regularization mechanism rather than a constraint that must be enforced at every optimization step. Once a coherent latent ordering is established, the model is able to preserve and refine this structure over multiple updates without requiring frequent reordering. Excessively frequent reordering may introduce unnecessary perturbations to the latent representation, while infrequent reordering remains sufficient to maintain global coherence.

\paragraph{Practical Implications.}
This behavior is favorable for 3D geometric operator learning, where ordering over latent anchors adds overhead but need not be performed at every optimization step. A moderate refresh interval, such as every 100 iterations in our experiments, preserves the accuracy gains of ordering while avoiding unnecessary reordering. More broadly, the weak sensitivity across a wide range of intervals indicates that Advectra's ordering mechanism behaves as a stable geometric prior rather than a fragile hyperparameter, making it practical for large-scale geometric problems.

\section{Additional Analysis of Geometric Ordering}
\label{app:ordering_additional}

\subsection{Robustness to Ordering Initialization}
\label{app:ordering_initialization}

We first evaluate whether the geometry-aware latent ordering is sensitive to the initial point of the greedy nearest-neighbor traversal. This experiment is conducted on ShapeNet-Car. To isolate the effect of ordering initialization, we fix the training seed and vary only the initial anchor used by the traversal.

Across eight different initial anchors, the greedy nearest-neighbor ordering achieves \(0.0148 \pm 0.0005\) volume error and \(0.0717 \pm 0.0005\) surface error. These values are close to the default reported result of \(0.0149\) volume error and \(0.0716\) surface error. Thus, the benefit of geometric ordering is not driven by a particular starting anchor, and the traversal remains stable under different initialization choices.

\subsection{Comparison with Alternative Traversals}
\label{app:ordering_alternatives}

We next compare the default greedy nearest-neighbor traversal with alternative latent ordering strategies. All methods are evaluated on ShapeNet-Car using the transported anchor geometry.

\begin{table}[H]
\centering
\caption{Effect of latent ordering strategy on the \textbf{ShapeNet-Car} dataset. We report relative \(L_2\) errors on volume and surface fields. Greedy nearest-neighbor ordering performs best among the tested traversal strategies.}
\label{tab:ordering_strategy}
\begin{sc}
\resizebox{0.58\textwidth}{!}{%
\begin{tabular}{l|cc}
\toprule
\textbf{Ordering Strategy} & \multicolumn{2}{c}{\textbf{ShapeNet-Car}} \\
\cmidrule(lr){2-3}
 & Volume $\downarrow$ & Surf $\downarrow$ \\
\midrule
Greedy nearest-neighbor & \textbf{0.0149} & \textbf{0.0716} \\
Random-walk traversal & 0.0159 & 0.0734 \\
MST traversal & 0.0166 & 0.0743 \\
\midrule
No sorting & 0.0169 & 0.0766 \\
\bottomrule
\end{tabular}
}
\end{sc}
\end{table}

The MST variant constructs a minimum spanning tree over transported anchors and then linearizes the tree into a sequence. Although the MST gives a short global connection graph, it is not itself a sequence. When linearized, the traversal can move between different branches of the tree, so consecutive sequence positions may correspond to anchors that are not locally adjacent in the original tree geometry. This mismatch is important for an SSM backbone, where adjacent tokens define the local direction along which sequential updates are applied.

The random-walk variant constructs an auxiliary graph over transported anchors and uses multiple sampled random-walk passes, following the spirit of Graph Mamba-style random-walk tokenization~\citep{behrouz2024graph}. This provides a graph-based traversal alternative, but it is less directly matched to Advectra's setting. In Graph Mamba, the graph is native: tokens are graph nodes or sampled subgraphs. In Advectra, tokens are learned latent anchors produced by source projection, so the graph used for random walks must first be constructed over the learned anchor set. Multiple sampled walks also introduce additional ordering overhead and may average over several local traversal directions rather than producing one deterministic spatially coherent path. Empirically, random-walk traversal improves over no sorting but remains less accurate than the default greedy nearest-neighbor ordering.

Overall, these results support the use of greedy nearest-neighbor ordering as a simple deterministic default: it directly constructs a locally coherent path over transported anchors, is robust to the initial anchor, and avoids the additional graph traversal overhead of MST or random-walk alternatives.

\section{Out-of-Distribution Generalization}
\label{sec:ood_results}

We evaluate the extrapolation capabilities of different models on the AirfRANS benchmark under two out-of-distribution (OOD) settings: unseen Reynolds numbers and unseen angles of attack. Results are reported in Table~\ref{tab:ood_AirfRANS}.

\textbf{Robust Generalization Across Flow Regimes.}
OOD Reynolds extrapolation requires models to generalize across different flow regimes, where both the scale and structure of the solution can change significantly. In this setting, Advectra achieves the lowest error (0.1807) and the highest rank correlation (0.9970), outperforming the evaluated baselines on this split. Strong neural operator baselines such as LaMO (0.2713) and Transolver (0.2996) exhibit noticeably larger errors, while several graph-based methods degrade substantially. This suggests that Advectra learns features that transfer across these Reynolds-number shifts, rather than relying only on regime-specific spatial correlations.

\textbf{Generalization to Unseen Geometric Conditions.}
For OOD angles of attack, which test generalization across boundary conditions and geometry-dependent flow behavior, Advectra again achieves the best performance, with the lowest error ($0.0932$) and highest rank correlation ($0.9975$). Compared to the strongest baselines (LinearNO and SpiderSolver), the improvement is consistent across both metrics, indicating stable behavior under distribution shift.

\textbf{Implications for Transport Modeling.}
These results highlight the role of transport-aware representations in OOD generalization. Extrapolation across Reynolds numbers and geometric conditions benefits from representations that account for how physical quantities change across regimes, rather than relying on fixed spatial correlations. By combining transported latent representations with a structured latent sequence, Advectra learns a more stable mapping that generalizes beyond the training distribution.

\begin{table}[t]
\centering
\caption{Generalization performance on Out-of-Distribution (OOD) tasks for the AirfRANS dataset. We evaluate extrapolation to unseen Reynolds numbers and angles of attack. $\downarrow$ indicates lower is better; $\uparrow$ indicates higher is better (Spearman's rank correlation). Top-2 results in each column are highlighted: best (orange) and second (blue).}
\label{tab:ood_AirfRANS}
\begin{sc}
\resizebox{0.6\textwidth}{!}{%
\begin{tabular}{l|cc|cc}
\toprule
\multicolumn{1}{c|}{\multirow{2}{*}{\textbf{Model}}} 
& \multicolumn{2}{c|}{\textbf{OOD Reynolds}} 
& \multicolumn{2}{c}{\textbf{OOD Angles}} \\
\multicolumn{1}{c|}{} 
& $C_L$ $\downarrow$ & $\rho_L$ $\uparrow$ 
& $C_L$ $\downarrow$ & $\rho_L$ $\uparrow$ \\ 
\midrule

Simple MLP & 0.6205 & 0.9578 & 0.4128 & 0.9572 \\
GraphSAGE \citep{hamilton2017inductive} & 0.4333 & 0.9707 & 0.2538 & 0.9894 \\
PointNet \citep{qi2017pointnet} & 0.3836 & 0.9806 & 0.4425 & 0.9784 \\
Graph U-Net \citep{gao2019graph} & 0.4664 & 0.9645 & 0.3756 & 0.9816 \\
MeshGraphNet \citep{pfaff2020learning} & 1.7718 & 0.7631 & 0.6525 & 0.8927 \\ 
\midrule

GNO \citep{li2020neural} & 0.4408 & 0.9878 & 0.3038 & 0.9884 \\
Galerkin \citep{cao2021choose} & 0.4615 & 0.9826 & 0.3814 & 0.9821 \\
GNOT \citep{hao2023gnot} & 0.3268 & 0.9865 & 0.3497 & 0.9868 \\
GINO \citep{li2023geometry} & 0.4180 & 0.9645 & 0.2583 & 0.9923 \\
Transolver \citep{wu2024transolver} & 0.2996 & 0.9896 & 0.1500 & 0.9950 \\
Transolver++ \citep{luotransolver++} & 0.2897 & 0.9856 & 0.2269 & 0.9914 \\
LaMO \citep{tiwarilatent} & 0.2713 & 0.9887 & 0.1505 & 0.9958 \\
SpiderSolver \citep{qispidersolver} & \cellcolor{ranktwo}0.2291 & 0.9922 & 0.1062 & 0.9941 \\
LinearNO \citep{hu2025transolver} & 0.2400 & \cellcolor{ranktwo}0.9951 & \cellcolor{ranktwo}0.0987 & \cellcolor{ranktwo}0.9963 \\ 
\midrule

\textbf{Advectra (Ours)} 
& \cellcolor{rankone}\textbf{0.1807}
& \cellcolor{rankone}\textbf{0.9970} 
& \cellcolor{rankone}\textbf{0.0932} 
& \cellcolor{rankone}\textbf{0.9975} \\

\bottomrule
\end{tabular}%
}
\end{sc}
\end{table}

\section{Transported Latent Operator Interpretation}
\label{app:transport_theory}

In this section, we provide a mathematical interpretation of Advectra as a transported finite-dimensional latent operator. The source assignment functions define integral measurements of the input field, the latent sequence model evolves the resulting finite-dimensional coefficients, and the target assignment functions reconstruct the output in a learned target frame.

Unlike a linear integral operator, the full Advectra map need not admit an input-independent kernel, since the latent dynamics include ordering, state-space updates, nonlinear activations, and channel mixing. The exact statement is instead a nonlinear encoder--latent--decoder factorization. This factorization explains the role of asymmetric source and target frames while retaining the full nonlinear architecture used in the main model. We then show that, in special cases, this formulation recovers the familiar finite-rank kernel view and the pushforward form of pure transport.

\subsection{Advectra as a Transported Finite-Dimensional Latent Operator}
\label{app:transport_factorization}

Let $\Omega \subset \mathbb{R}^m$ be a physical domain, let
$u:\Omega \to \mathbb{R}^{d_u}$ denote an input field, and let $\rho$ be a sampling density over $\Omega$.
We denote by
\[
    \{\psi_h\}_{h=1}^G,
    \qquad
    \{\phi_g\}_{g=1}^G
\]
the source and target assignment functions, respectively. The source functions $\psi_h$ aggregate input information, while the target functions $\phi_g$ reconstruct the output in the target coordinate frame.

For simplicity of exposition, we first consider scalar-valued fields. The vector-valued case follows componentwise, with the latent sequence model acting on channel-valued latent tokens. Define the source encoder
\[
    (E_\psi u)_h
    =
    z_h(u)
    =
    \frac{
        \int_{\Omega} \psi_h(x)u(x)\rho(x)\,dx
    }{
        \int_{\Omega} \psi_h(x)\rho(x)\,dx
    },
    \qquad h=1,\ldots,G,
\]
assuming the denominators are nonzero. Let
\[
    F_\theta:\mathbb{R}^{G}\to\mathbb{R}^{G}
\]
denote the latent sequence model, including geometric ordering, SSM updates, nonlinear activations, and channel mixing. Finally, define the target decoder
\[
    (D_\phi a)(y)
    =
    \sum_{g=1}^G \phi_g(y)a_g .
\]

\begin{proposition}[Exact transported latent factorization]
\label{prop:latent_factorization}
For fixed source and target assignments, the Advectra reconstruction can be written exactly as
\[
    \hat u
    =
    D_\phi \circ F_\theta \circ E_\psi (u),
\]
or equivalently,
\[
    \hat u(y)
    =
    \sum_{g=1}^G
    \phi_g(y)
    \left[
        F_\theta(z_1(u),\ldots,z_G(u))
    \right]_g .
\]
Thus, Advectra maps the input field through a finite-dimensional latent representation and reconstructs the output in the span of the target assignment functions. The full map may be nonlinear through $F_\theta$, but for fixed source and target assignments, all field-value dependence enters through the finite set of source coefficients $\{z_h(u)\}_{h=1}^G$.
\end{proposition}

\begin{proof}
By definition, the source assignment functions produce latent coefficients
\[
    z(u)=E_\psi u \in \mathbb{R}^G.
\]
The latent sequence model maps these coefficients to
\[
    a=F_\theta(z(u))\in \mathbb{R}^G.
\]
The target assignment functions reconstruct the output as
\[
    \hat u(y)
    =
    (D_\phi a)(y)
    =
    \sum_{g=1}^G \phi_g(y)a_g
    =
    \sum_{g=1}^G
    \phi_g(y)
    \left[
        F_\theta(z(u))
    \right]_g .
\]
This proves the factorization.
\end{proof}

\paragraph{Linear special case.}
If the latent dynamics are linear, i.e. $F_\theta(z)=Az$ for some $A\in\mathbb{R}^{G\times G}$, then Proposition~\ref{prop:latent_factorization} reduces to a finite-rank integral operator:
\[
    \hat{u}(y)
    =
    \int_{\Omega} K_A(y,x)u(x)\rho(x)\,dx,
\]
where
\[
    K_A(y,x)
    =
    \sum_{g=1}^G\sum_{h=1}^G
    \phi_g(y)A_{gh}
    \frac{\psi_h(x)}
    {\int_{\Omega}\psi_h(s)\rho(s)\,ds}.
\]
Thus, the finite-rank kernel expression is a special case of the nonlinear transported latent operator, not an assumption on the full Advectra model.

\paragraph{Interpretation.}
The source assignments define localized measurements of the input field, while the target assignments define reconstruction functions in the output coordinate frame. The latent dynamics $F_\theta$ couple and transform these finite-dimensional coefficients. This provides a nonlinear transported latent-operator view of Advectra: the model aggregates in a source frame, evolves a finite-dimensional latent representation, and reconstructs in a target frame that need not coincide with the source frame.

\subsection{Locally Invertible Learned Coordinate Frames}
\label{app:local_invertibility}

Advectra constructs source and target coordinates using learned displacement maps. For a fixed context variable \(s\), write a coordinate update as
\[
    \tau(x)=x+\mathcal{T}(x,s).
\]
Under the scaled Lipschitz-constrained parameterization in Sec.~\ref{sec:tcs}, the displacement satisfies
\[
    \|\nabla_x \mathcal{T}(x,s)\|_{\mathrm{op}}<1
\]
on an open neighborhood of the coordinate domain used by the operator.

\begin{proposition}[Local invertibility of the learned coordinate map]
\label{prop:local_invertibility}
Assume \(\mathcal{T}(\cdot,s)\) is continuously differentiable on an open set containing the coordinate domain and
\[
    \|\nabla_x \mathcal{T}(x,s)\|_{\mathrm{op}} < 1
\]
for all \(x\) in the domain. Then the coordinate update
\[
    \tau(x)=x+\mathcal{T}(x,s)
\]
is locally invertible. Moreover, if \(\mathcal{T}(\cdot,s)\) is globally Lipschitz with constant \(L<1\) on the coordinate domain, then \(\tau\) is bi-Lipschitz onto its image:
\[
    (1-L)\|x-y\|
    \le
    \|\tau(x)-\tau(y)\|
    \le
    (1+L)\|x-y\|
\]
for all \(x,y\) in the domain.
\end{proposition}

\begin{proof}
The Jacobian of \(\tau\) is
\[
    \nabla \tau(x)=I+\nabla_x \mathcal{T}(x,s).
\]
Since \(\|\nabla_x \mathcal{T}(x,s)\|_{\mathrm{op}}<1\), the matrix
\(I+\nabla_x \mathcal{T}(x,s)\) is nonsingular by the Neumann-series criterion. Therefore \(\tau\) is locally invertible by the inverse function theorem.

If \(\mathcal{T}(\cdot,s)\) is globally Lipschitz with constant \(L<1\), then for any \(x,y\) in the coordinate domain,
\[
    \|\tau(x)-\tau(y)\|
    =
    \|x-y+\mathcal{T}(x,s)-\mathcal{T}(y,s)\|.
\]
The upper bound follows from the triangle inequality:
\[
    \|\tau(x)-\tau(y)\|
    \le
    (1+L)\|x-y\|.
\]
The lower bound follows from the reverse triangle inequality:
\[
    \|\tau(x)-\tau(y)\|
    \ge
    \|x-y\|-\|\mathcal{T}(x,s)-\mathcal{T}(y,s)\|
    \ge
    (1-L)\|x-y\|.
\]
Hence \(\tau\) is injective and bi-Lipschitz onto its image.
\end{proof}

\paragraph{Interpretation.}
This constraint prevents local folding of the learned coordinate map and supports interpreting the source and target frames as valid local coordinate systems. It does not require the learned map to equal the true physical flow map. Rather, it makes the learned frame a locally valid, approximately co-moving coordinate frame trained to align transported structures.

\subsection{Pure Transport as Asymmetric Assignment}
\label{app:pure_transport}

We next isolate the role of asymmetric source and target assignments. This result does not require a linear latent SSM in the full model. Instead, it shows that when the latent dynamics are the identity, the assignment geometry alone can represent pure transport, provided the target assignments align with transported source assignments.

Let $T:\Omega\to\Omega$ be an invertible transport map. The associated pushforward action on scalar fields is
\[
    (\mathcal{T}u)(y)
    =
    u(T^{-1}(y)).
\]

\begin{proposition}[Pure transport from source-target assignment]
\label{prop:pure_transport}
Let $\{\psi_g\}_{g=1}^G$ be a family of localized source functions, and define the transported target functions
\[
    \phi_g(y)=\psi_g(T^{-1}(y)).
\]
Suppose the latent dynamics are the identity, $F_\theta=I$. Then the Advectra reconstruction becomes
\[
    \hat{u}(y)
    =
    \sum_{g=1}^G
    \psi_g(T^{-1}(y))z_g(u).
\]
If the source reconstruction
\[
    \Pi_G u(x)
    =
    \sum_{g=1}^G \psi_g(x) z_g(u)
\]
converges to $u(x)$ as $G\to\infty$, then
\[
    \hat{u}(y)
    \to
    u(T^{-1}(y)).
\]
Thus, asymmetric source and target assignments have the capacity to approximate the transport operator induced by $T$ when the target basis is aligned with the transported source basis.
\end{proposition}

\begin{proof}
With $F_\theta=I$, Proposition~\ref{prop:latent_factorization} gives
\[
    \hat{u}(y)
    =
    \sum_{g=1}^G \phi_g(y) z_g(u).
\]
Substituting $\phi_g(y)=\psi_g(T^{-1}(y))$, we obtain
\[
    \hat{u}(y)
    =
    \sum_{g=1}^G \psi_g(T^{-1}(y))z_g(u)
    =
    \Pi_Gu(T^{-1}(y)).
\]
By assumption, $\Pi_Gu(x)\to u(x)$ as $G\to\infty$. Evaluating this convergence at $x=T^{-1}(y)$ yields
\[
    \hat{u}(y)
    =
    \Pi_Gu(T^{-1}(y))
    \to
    u(T^{-1}(y)).
\]
This is precisely the transported field $(\mathcal{T}u)(y)$.
\end{proof}

\paragraph{Connection to discrete implementation.}
In Advectra, $\psi_g$ and $\phi_g$ correspond to normalized assignment weights constructed from similarity kernels in the learned source and target frames. Proposition~\ref{prop:pure_transport} should be read as an idealized representational statement. The implementation does not enforce the identity $\phi_g = \psi_g \circ T^{-1}$; instead, $\psi_g$ and $\phi_g$ are softmaxes over similarity kernels evaluated in two separately learned frames that are constructed from the same transport networks $\mathcal{T}$ and $\mathcal{U}$ applied at different coordinate states. The proposition therefore characterizes the representational capacity of the asymmetric assignment construction: when training drives the learned target frame to align with the transported source frame, and when the source assignments form a convergent reconstruction family, asymmetric reconstruction can approximate finite-time pushforward transport. Whether training achieves this alignment in practice is an empirical question, addressed by the transport diagnostics in Appendix~\ref{app:transport_analysis}. The learned latent dynamics $F_\theta$ in the full model then operate on these transported latent coefficients and can model effects beyond pure transport, such as deformation, diffusion, forcing, and nonlinear feature interaction.

\paragraph{Interpretation.}
This proposition formalizes why Advectra uses distinct source and target assignments. The source assignment extracts information in the input frame, while the target assignment deposits the latent representation in a transported output frame. When the target assignment aligns with the transported source basis, the architecture can represent the pushforward form of finite-time transport.

\subsection{Approximately Co-Moving Frame Interpretation}
\label{app:comoving}

We now relate the transported-assignment construction to co-moving coordinates. Let $T_t:\Omega\to\Omega$ denote a flow map. A passively transported scalar satisfies
\[
    u_t(y)=u_0(T_t^{-1}(y)).
\]

\begin{proposition}[Stationarity in exact co-moving coordinates]
\label{prop:comoving_exact}
Suppose a scalar field evolves by pure transport under an invertible map $T_t$, i.e.,
\[
    u_t(y)=u_0(T_t^{-1}(y)).
\]
Then in co-moving coordinates $y=T_t(x)$, the field is stationary:
\[
    u_t(T_t(x)) = u_0(x).
\]
\end{proposition}

\begin{proof}
Substitute $y=T_t(x)$ into the transport relation:
\[
    u_t(T_t(x))
    =
    u_0(T_t^{-1}(T_t(x))).
\]
Since $T_t$ is invertible,
\[
    T_t^{-1}(T_t(x))=x.
\]
Therefore,
\[
    u_t(T_t(x))=u_0(x),
\]
which proves the claim.
\end{proof}

The learned coordinate map in Advectra is not assumed to equal the true flow map. The following elementary bound explains the sense in which a learned locally invertible map can define an approximately co-moving frame.

\begin{proposition}[Approximate co-moving coordinates]
\label{prop:approx_comoving}
Suppose a scalar field evolves by pure transport under an invertible flow map $T_t$:
\[
    u_t(y)=u_0(T_t^{-1}(y)).
\]
Let $\tau_t$ be a learned locally invertible coordinate map satisfying $\tau_t(x)\in T_t(\Omega)$ for all $x\in\Omega$. Assume $T_t^{-1}$ is $L_T$-Lipschitz on $T_t(\Omega)$, $u_0$ is $L_u$-Lipschitz, and
\[
    \sup_{x\in\Omega}\|\tau_t(x)-T_t(x)\|\le \varepsilon.
\]
Then
\[
    \sup_{x\in\Omega}
    |u_t(\tau_t(x))-u_0(x)|
    \le
    L_u L_T \varepsilon.
\]
\end{proposition}

\begin{proof}
Using the transport relation,
\[
    u_t(\tau_t(x))
    =
    u_0(T_t^{-1}(\tau_t(x))).
\]
Therefore,
\[
    |u_t(\tau_t(x))-u_0(x)|
    =
    |u_0(T_t^{-1}(\tau_t(x)))-u_0(T_t^{-1}(T_t(x)))|.
\]
Since $u_0$ is $L_u$-Lipschitz and $T_t^{-1}$ is $L_T$-Lipschitz,
\[
    |u_t(\tau_t(x))-u_0(x)|
    \le
    L_u L_T \|\tau_t(x)-T_t(x)\|
    \le
    L_u L_T \varepsilon.
\]
Taking the supremum over $x\in\Omega$ gives the claim.
\end{proof}

\paragraph{Interpretation.}
Proposition~\ref{prop:comoving_exact} gives the exact co-moving identity for a true transport map. Proposition~\ref{prop:approx_comoving} formalizes the approximate case as a conditional statement: if a learned coordinate map is $\varepsilon$-close to the physical transport map, then the field becomes $L_u L_T \varepsilon$-stationary in the learned coordinates. The training objective does not directly minimize $\varepsilon$, so this proposition does not assert that the learned map approximates the true flow map. It characterizes the sense in which Advectra's source and target frames could be interpreted as co-moving: under alignment of the learned map with transport, the field is approximately stationary in the learned coordinates. Whether training drives such alignment is again an empirical question, supported by the center-of-mass and structural-fidelity diagnostics in Appendix~\ref{app:transport_analysis}.

\subsection{Ordering and Locality of Sequential Latent Updates}
\label{app:ordering}

Finally, we formalize the purpose of geometric ordering for the SSM backbone. State-space sequence models apply updates along a one-dimensional ordering. If neighboring sequence indices correspond to distant physical regions, local sequential interactions may connect physically unrelated regions.

Let $m_g\in\Omega$ denote the physical centroid or representative location of latent token $g$. For an ordering $\pi$ of $\{1,\ldots,G\}$, define the cumulative physical jump length
\[
    L(\pi)
    =
    \sum_{k=1}^{G-1}
    \|m_{\pi(k+1)}-m_{\pi(k)}\|.
\]

\begin{proposition}[Ordering reduces adjacent physical jump length]
\label{prop:ordering}
Consider a sequential operator whose strongest local interactions occur between neighboring sequence positions, e.g. $k-1,k,k+1$. Under an ordering $\pi$, the total physical length of adjacent sequence interactions is
\[
    L(\pi)
    =
    \sum_{k=1}^{G-1}
    \|m_{\pi(k+1)}-m_{\pi(k)}\|.
\]
Therefore, choosing an ordering with smaller $L(\pi)$ reduces the cumulative physical distance across adjacent sequence interactions.
\end{proposition}

\begin{proof}
Under ordering $\pi$, neighboring sequence positions $k$ and $k+1$ correspond to physical token locations $m_{\pi(k)}$ and $m_{\pi(k+1)}$. The physical distance associated with this adjacent sequence interaction is
\[
    d_k = \|m_{\pi(k+1)}-m_{\pi(k)}\|.
\]
The total physical extent of all adjacent sequence couplings is therefore
\[
    \sum_{k=1}^{G-1} d_k
    =
    \sum_{k=1}^{G-1}
    \|m_{\pi(k+1)}-m_{\pi(k)}\|
    =
    L(\pi).
\]
Thus, reducing $L(\pi)$ directly reduces the cumulative distance across adjacent sequence interactions. This proves the claim.
\end{proof}

\paragraph{Interpretation.}
This result does not require the ordering to solve a globally optimal traveling-salesperson problem. It only requires that the ordering reduce the physical discontinuity between adjacent latent tokens. This explains why geometry-aware ordering can improve SSM-based operators on complex 3D geometries: it aligns the one-dimensional latent sequence with spatially coherent neighborhoods, reducing the extent to which local sequential interactions connect physically distant regions.

\paragraph{Summary.}
Together, Propositions~\ref{prop:latent_factorization}--\ref{prop:ordering} provide a mathematical interpretation of Advectra as a transported finite-dimensional latent operator. The source assignments define integral measurements of the input field, the nonlinear latent SSM evolves these finite-dimensional coefficients, and the target assignments reconstruct the output in a learned target frame. The Jacobian constraint makes the learned coordinate updates locally invertible, supporting the interpretation of the source and target frames as candidate co-moving latent coordinate systems whose alignment with transport is encouraged by training and probed empirically in Appendix~\ref{app:transport_analysis}. In the special case where target assignments coincide with transported source assignments, asymmetric reconstruction recovers the standard pushforward form of pure transport. Finally, geometric ordering reduces adjacent physical jump length in the latent sequence, encouraging spatially coherent SSM updates.

\section{Runtime Efficiency and Computational Overhead}
\label{app:runtime}

We evaluate the runtime efficiency of Advectra compared to prior latent operator baselines on the Airfoil benchmark. A primary concern of the proposed method is whether the additional ordering mechanism introduces non-negligible computational overhead.

\paragraph{Empirical Runtime Comparison.}
Table~\ref{tab:runtime} reports the average wall-clock time per training iteration and the corresponding floating-point operations (FLOPs) per batch. Despite incorporating an explicit ordering step, Advectra achieves runtime comparable to existing methods, with only a marginal increase in iteration time.

\begin{table}[h]
\centering
\caption{Runtime comparison on the Airfoil benchmark. We report average time per iteration and FLOPs per batch.}
\label{tab:runtime}
\begin{tabular}{l|cc}
\toprule
\textbf{Model} & \textbf{Time (s/iter)} $\downarrow$ & \textbf{GFLOPs / batch} $\downarrow$ \\
\midrule
Transolver & 0.1111 & 266.09 \\
LaMO & 0.1150 & 266.69 \\
\textbf{Advectra (Ours)} & 0.1209 & 258.06 \\
\bottomrule
\end{tabular}
\end{table}

\paragraph{Discussion.}
Although Advectra introduces an additional ordering step over latent tokens, this does not lead to a significant increase in computational cost in practice. All compared methods operate on a compressed set of latent tokens, and Advectra achieves strong performance with a relatively small number of tokens.

Under the reported Airfoil configuration, Advectra exhibits slightly lower FLOPs than both Transolver and LaMO, indicating that the overall computational cost is dominated by the latent transformation rather than the ordering step.

These results suggest that the proposed ordering mechanism incurs modest practical overhead in this benchmark while contributing to the observed accuracy improvements. In particular, effective performance with fewer tokens reduces the computational burden on the backbone, offsetting the cost of ordering.

\subsection{Computational Complexity Analysis}
\label{app:complexity}

We analyze the computational cost of Advectra and show that the proposed transport and ordering mechanisms do not change the dominant scaling of latent operator layers under the standard latent-operator regime. Let \(M\) denote the number of spatial mesh points, \(G\) the number of latent anchors, \(d\) the embedded coordinate dimension, and \(D\) the feature channel dimension. In typical latent neural operators \citep{wu2024transolver, luotransolver++, tiwarilatent, hu2025transolver}, the spatial resolution is much larger than the latent resolution:
\[
    M \gg G .
\]

\paragraph{Transported coordinates and assignments.}
Advectra constructs source and target coordinate systems
\[
    (X_s,C_s), \qquad (X_t,C_t),
\]
where \(X_s,X_t\in\mathbb{R}^{M\times d}\) and \(C_s,C_t\in\mathbb{R}^{G\times d}\). The learned displacement maps are evaluated pointwise on spatial coordinates and anchors, so their cost is linear in \(M+G\) and does not introduce any pairwise \(M^2\) interaction.

Advectra then computes separate source and target assignment matrices,
\[
    \Psi=\mathcal{A}(X_s,C_s),
    \qquad
    \Phi=\mathcal{A}(X_t,C_t).
\]
Each assignment evaluates a similarity kernel between \(M\) embedded points and \(G\) anchors. Let \(c_\kappa(d)\) denote the cost of one kernel evaluation in the embedded coordinate space. Then one assignment matrix costs
\[
    O\!\left(MG\,c_\kappa(d)\right),
\]
plus \(O(MG)\) for normalization over anchors. Computing both \(\Psi\) and \(\Phi\) changes only the constant factor.

For dot-product or distance-based kernels, \(c_\kappa(d)=O(d)\). For an MLP-based similarity kernel, \(c_\kappa(d)\) is the cost of the small per-pair MLP applied to the point-anchor representation. In our implementation this kernel is lightweight relative to the feature projection width, so the assignment stage remains
\[
    O\!\left(MG\,c_\kappa(d)\right)
\]
and does not dominate the \(D\)-dimensional feature aggregation and reconstruction costs.

\paragraph{Projection, latent processing, and reconstruction.}
The source projection aggregates \(D\)-dimensional mesh features into \(G\) latent tokens:
\[
    Z_g =
    \frac{\sum_{i=1}^M \Psi_{ig} V_i}
    {\sum_{i=1}^M \Psi_{ig}},
\]
which costs \(O(MGD)\). The target reconstruction
\[
    Y_i=\sum_{g=1}^G \Phi_{ig} Z'_g
\]
also costs \(O(MGD)\). Thus, the mesh-to-token and token-to-mesh stages have dominant feature-routing cost
\[
    O(MGD).
\]

The latent SSM backbone operates only on the compressed sequence of length \(G\), independent of \(M\). For a linear-time SSM layer, this costs
\[
    O(GD),
\]
or \(O(LGD)\) for \(L\) layers.

\paragraph{Ordering cost.}
Advectra additionally orders the \(G\) anchors according to their transported geometry. This operation is performed only over latent anchors and is independent of \(M\). A nearest-neighbor traversal over anchor distances costs
\[
    O(G^2).
\]
Since ordering is refreshed every \(K\) iterations and does not require backpropagation, the amortized per-iteration cost is
\[
    O(G^2/K).
\]

\begin{proposition}[Asymptotic complexity of Advectra]
\label{prop:advectra_complexity}
Assume the standard latent-operator regime \(M\gg G\), a fixed ordering refresh interval \(K\), and a lightweight assignment kernel satisfying \(c_\kappa(d)=O(D)\) under the chosen model configuration. Then the amortized per-layer complexity of Advectra is
\[
    O(MGD),
\]
matching standard latent routing up to constant factors.
\end{proposition}

\begin{proof}
Combining the terms above, the amortized per-layer complexity is
\[
    \mathcal{C}_{\mathrm{Advectra}}
    =
    O\!\left(MG\,c_\kappa(d)\right)
    +
    O(MGD)
    +
    O(GD)
    +
    O(G^2/K).
\]
The first term is the coordinate-kernel assignment cost, the second is feature projection and reconstruction, the third is latent SSM processing, and the fourth is amortized ordering.

Under the lightweight-kernel assumption \(c_\kappa(d)=O(D)\), the assignment term is no larger than \(O(MGD)\). Under \(M\gg G\), the spatial routing term also dominates \(O(GD)\) and \(O(G^2/K)\). Therefore,
\[
    \mathcal{C}_{\mathrm{Advectra}} = O(MGD).
\]
\end{proof}

\paragraph{Discussion.}
The transport mechanism introduces separate source and target assignments, which affects constant factors but not the dominant scaling. The assignment kernels operate in embedded coordinate space; whether instantiated as dot-product, distance-based, or lightweight MLP kernels, their cost is captured by \(O(MG\,c_\kappa(d))\). The dominant feature aggregation and reconstruction costs scale as \(O(MGD)\), while ordering is computed only over the compressed anchor set and amortized across iterations. Consequently, Advectra preserves the scalability of latent neural operators while adding transport-aware source and target charts.

\section{Scalability with Respect to Token Count}
\label{app:anchor_scaling}


We analyze token-count sensitivity on Airfoil, where varying the number of latent slices is a standard efficiency study. This experiment is complementary to the main industrial comparisons, where all baselines use their default best-reported configurations from prior work.

\begin{figure}[t]
    \centering
    \includegraphics[width=0.7\textwidth]{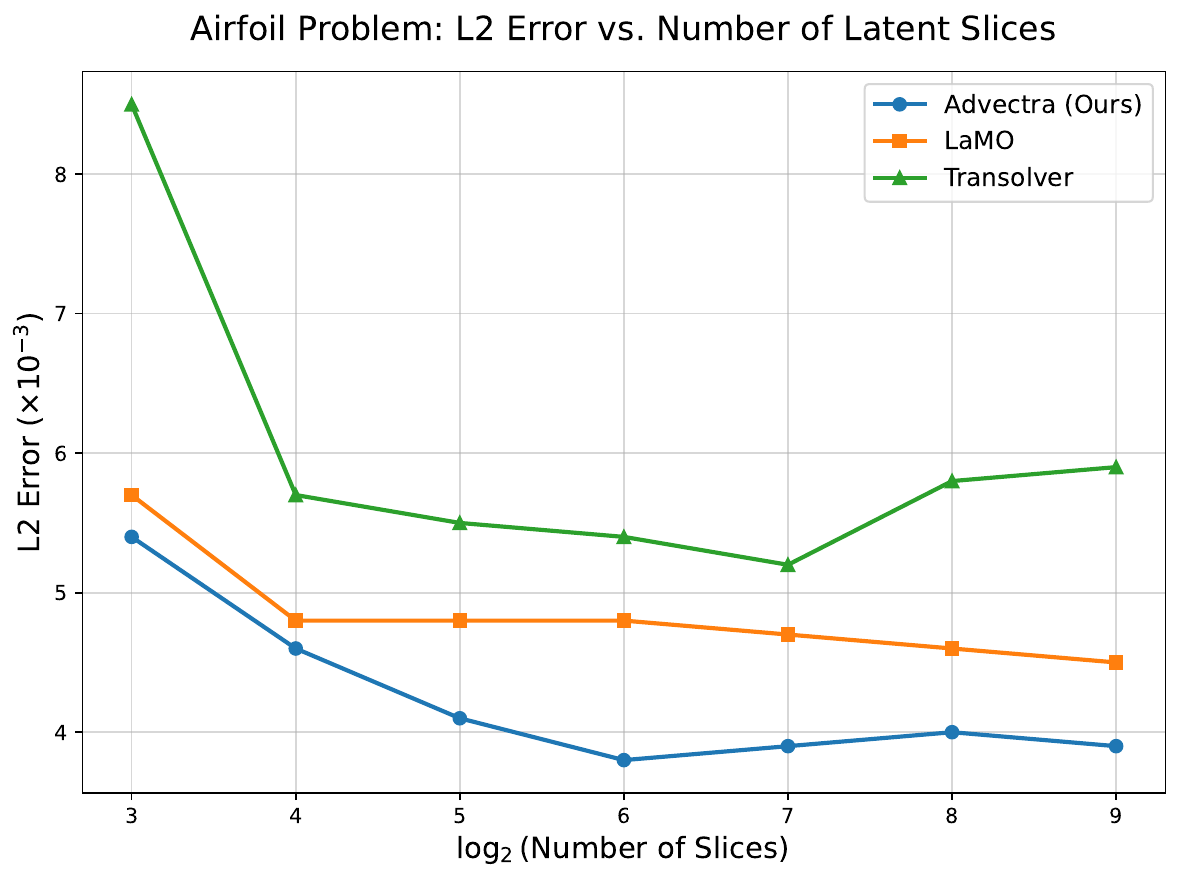}
    \caption{Relative $L_2$ error on the Airfoil benchmark as a function of token count. All methods improve with increasing token count and eventually saturate. Advectra reaches its saturated regime at a lower token count in this sweep, while the stronger \(128\)-slice setting benefits Transolver and LaMO.}
    \label{fig:anchor_scaling}
\end{figure}

Figure~\ref{fig:anchor_scaling} presents the relative $L_2$ error on the Airfoil benchmark as a function of token count. We observe that all methods benefit from increasing the number of tokens and converge to similar performance regimes at higher resolutions (around $128$ tokens and beyond). Improvements beyond this point are on the order of $10^{-4}$ and are therefore marginal.


A key distinction lies in how quickly each method reaches its saturated regime. Advectra saturates at approximately \(32\)--\(64\) tokens, while Transolver and LaMO continue improving up to the stronger \(128\)-slice setting in this sweep.

These results suggest that Advectra uses latent tokens more efficiently than prior latent operator architectures. Rather than relying primarily on increased token resolution, the transported source-target construction enables each token to capture more informative geometric and dynamical structure.

\paragraph{Interpretation.}
These results suggest that Advectra uses latent tokens more efficiently in the Airfoil token-count sweep. This does not alter the main industrial comparison protocol: on AirfRANS, ShapeNet-Car, and AirCraft, all baselines are evaluated using their default best-reported configurations, while Advectra uses the same latent-token budget with a smaller channel width.

\paragraph{Discussion.}
This behavior is advantageous in practice, since increasing token count directly impacts memory and computational cost. The ability to reach strong performance with fewer tokens improves the efficiency-accuracy tradeoff of the latent representation, reducing the need for large token budgets while maintaining accuracy.

\section{Benchmarks}
\label{app:benchmarks}

We evaluate Advectra on transport-dominated PDE benchmarks and industrial aerodynamic design benchmarks. The PDE benchmarks are drawn from POSEIDON~\citep{herde2024POSEIDON}, where the emphasis is on finite-time prediction of non-stationary dynamics under a semigroup training protocol. The industrial benchmarks evaluate performance on structured and unstructured aerodynamic geometries, ranging from 2D airfoils to large-scale 3D automotive and aircraft simulations.

\subsection{POSEIDON PDE Benchmarks}

For the PDE experiments, we follow the semigroup-style operator learning formulation used in POSEIDON~\citep{herde2024POSEIDON}. Rather than training models autoregressively to predict only the next time step, the model is trained on input-output pairs across different temporal gaps:
\[
    u(t_j) = \mathcal{S}_{t_j-t_i}(u(t_i)),
\]
where $\mathcal{S}_{\Delta t}$ denotes the solution operator over a finite time interval $\Delta t$. This formulation evaluates whether a model can directly learn long-horizon evolution without iterative rollout.

In our experiments, we use $1000$ training trajectories and $130$ testing trajectories for all POSEIDON datasets. This differs from the original POSEIDON setting, which uses substantially larger trajectory collections for foundation-model pretraining. For GCE-RT, we evaluate prediction at the $7$th time step. For NS-Tracer-PwC and FNS-KF, we evaluate prediction at the $14$th time step, following the original POSEIDON evaluation horizon. For NS-Tracer-PwC and FNS-KF, we additionally downsample both space and time by a factor of $2$.

\paragraph{GCE-RT.}
The GCE-RT benchmark models the two-dimensional compressible Euler equations with gravitational forcing. The conserved hydrodynamic state is
\[
    u = [\rho, \rho v_x, \rho v_y, E]^\top,
\]
where \(\rho\) is density, \((v_x,v_y)\) is velocity, and \(E\) is total energy. The equations take the form
\[
    \partial_t u + \nabla \cdot F(u) = S_\phi,
\]
where \(F(u)\) is the compressible Euler flux and \(S_\phi\) contains gravitational source terms determined by the prescribed gravitational potential \(\phi\).

The released GCE-RT state also includes a passive tracer \(q\), which is carried by the flow and does not feed back into the Euler dynamics. We use this tracer as an additional evaluation variable for interface transport and mixing. In continuous notation, the tracer can be viewed as a material scalar satisfying
\[
    \partial_t q + v\cdot\nabla q = 0
\]
up to numerical diffusion in the simulator.

The benchmark is based on a stochastic Rayleigh--Taylor instability in a two-dimensional model neutron-star configuration. The spatial domain is
\[
    \Omega = [-1/2,1/2]^2.
\]
The pressure and gravitational potential are initialized from a polytropic equilibrium:
\[
    p(r) = K_0 \rho_0^2
    \left(\frac{\sin(\alpha r)}{\alpha r}\right)^2,
    \qquad
    \phi(r) =
    -2K_0\rho_0
    \frac{\sin(\alpha r)}{\alpha r},
\]
where \(r=\sqrt{x^2+y^2}\), \(K_0=p_0/\rho_0^2\), and \(\alpha\) is determined by the gravitational constant and polytropic constant. The initial velocity is zero. The density profile contains a Rayleigh--Taylor interface at
\[
    r_{\mathrm{RT}} = 0.25\big(1+a\cos(\operatorname{atan2}(y,x)+b)\big),
\]
where \(a\) and \(b\) are random variables controlling the perturbation amplitude and phase. Additional stochasticity is introduced through perturbations of the central density, pressure, and Atwood number.

The solution operator is evaluated on the evolving fields
\[
    \mathcal{S}(t,\rho^0,v_x^0,v_y^0,p^0,q^0;\phi)
    =
    [\rho(t),v_x(t),v_y(t),p(t),q(t)].
\]
Here \(\phi\) is treated as a prescribed static potential that determines the gravitational forcing, not as a dynamic prediction target. Accordingly, we omit \(\phi\) from the main error table and report GCE-RT performance only on the evolving hydrodynamic and tracer variables. This dataset is challenging because the gravitationally driven instability produces strong nonlinear mixing and multi-scale interface deformation.

\paragraph{NS-Tracer-PwC.}
NS-Tracer-PwC augments the incompressible Navier--Stokes equations with a passive scalar tracer. The velocity field evolves according to the Navier--Stokes dynamics, while the tracer concentration $c(x,y,t)$ is transported by the velocity field without influencing it. The tracer satisfies the advection-diffusion equation
\[
    \partial_t c + v\cdot \nabla c = \kappa \Delta c,
\]
where $v=(v_x,v_y)$ is the fluid velocity and $\kappa$ is the tracer diffusivity. The tracer is initialized as a localized sphere centered in the domain, while the stochasticity comes from the random initial velocity field used in the underlying Navier--Stokes simulation.

The corresponding solution operator is
\[
    \mathcal{S}(t,v_x^0,v_y^0,c^0)
    =
    [v_x(t),v_y(t),c(t)].
\]
This benchmark is especially relevant for Advectra because the tracer directly measures the ability of the model to represent finite-time advection, stretching, and diffusion. We report prediction error at the $14$th time step, corresponding to the long-horizon evaluation protocol in POSEIDON.

\paragraph{FNS-KF.}
FNS-KF considers the two-dimensional forced Navier--Stokes equations with a Kolmogorov-type forcing term:
\[
    \partial_t v + (v\cdot\nabla)v + \nabla p - \nu \Delta v = f,
    \qquad
    \nabla\cdot v = 0,
\]
on the periodic domain $[0,1]^2$. The forcing is time-independent and given by
\[
    f(x,y)=0.1\sin(2\pi(x+y)).
\]
The velocity field is initialized using the same distribution as the NS-Tracer-PwC family of tasks. Equivalently, the problem can be viewed as an augmented system with state
\[
    U=[v_x,v_y,f],
\]
where the forcing satisfies the trivial evolution equation $\partial_t f=0$. The solution operator is therefore
\[
    \mathcal{S}(t,U^0)
    =
    [v_x(t),v_y(t),f].
\]
This benchmark tests whether the model can capture non-stationary velocity evolution under persistent external forcing. We evaluate velocity prediction at the $14$th time step.

\subsection{Aerodynamic and Industrial Benchmarks}

\paragraph{Airfoil.}
The Airfoil benchmark evaluates transonic flow over 2D airfoil geometries governed by the Euler equations. The geometries are generated from deformations of the NACA-0012 airfoil and discretized on a structured mesh of size $221\times 51$. The model takes mesh node coordinates as input and predicts the Mach number field over the domain. We use $1000$ training samples and $200$ testing samples.

\paragraph{AirfRANS.}
AirfRANS is a high-fidelity 2D aerodynamic benchmark based on Reynolds-Averaged Navier--Stokes (RANS) simulations. The dataset contains airfoils from the NACA 4-digit and 5-digit families, with variations in airfoil geometry, angle of attack, and Reynolds number. Each sample is represented on an unstructured mesh with approximately $32{,}000$ points. We use $800$ samples for training and $200$ samples for testing.

The prediction targets include physical quantities on both the volume and airfoil surface, as well as the lift coefficient $C_L$. In addition to field-level relative error, we report the relative error of $C_L$ and Spearman's rank correlation $\rho_L$ to measure whether the model preserves the ordering of airfoil designs by lift.

\paragraph{ShapeNet-Car.}
ShapeNet-Car evaluates 3D automotive aerodynamics for vehicle geometries drawn from the car category of ShapeNet. The task is to model airflow around cars moving at $72$ km/h. The dataset contains $889$ samples, with $789$ used for training and $100$ used for testing. Each geometry is discretized into an unstructured mesh with $32{,}186$ points. The input features include spatial coordinates, signed distance function values, and surface normal vectors.

The model predicts pressure and velocity fields over the discretized domain. These predicted fields are then used to compute the drag coefficient $C_D$. We evaluate volume and surface field errors, relative drag coefficient error, and Spearman's rank correlation $\rho_D$ for aerodynamic design ranking.

\paragraph{AirCraft.}
AirCraft is a large-scale 3D external aerodynamics benchmark consisting of high-fidelity CFD simulations over more than $30$ aircraft designs. The dataset includes variations in flow conditions such as Mach number and angle of attack. Each aircraft is represented by a large unstructured mesh with approximately $300{,}000$ points. Each point contains geometric information including spatial coordinates and surface normals.

The model predicts six physical quantities:
\[
    C_p,\quad \rho,\quad U,\quad V,\quad W,\quad p,
\]
where $C_p$ is the pressure coefficient, $\rho$ is density, $(U,V,W)$ are the velocity components, and $p$ is pressure. Due to dataset availability, we use $110$ samples for training and $10$ samples for testing. This benchmark evaluates scalability to realistic 3D engineering geometries and multi-field prediction.

\subsection{Evaluation Metrics}

\paragraph{Relative Field Error.}
For all benchmarks, we report relative $L_2$ error for field prediction:
\[
    \mathrm{Rel}\text{-}L_2
    =
    \frac{\|\hat{u}-u\|_2}{\|u\|_2},
\]
where $u$ is the ground-truth physical field and $\hat{u}$ is the model prediction.

\paragraph{Aerodynamic Coefficient Error.}
For AirfRANS and ShapeNet-Car, we also evaluate scalar aerodynamic coefficients derived from the predicted fields. Specifically, we report lift coefficient error for AirfRANS and drag coefficient error for ShapeNet-Car:
\[
    \mathrm{Rel}\text{-}L_2(C)
    =
    \frac{\|\hat{C}-C\|_2}{\|C\|_2}.
\]
Here $C$ denotes either the lift coefficient $C_L$ or drag coefficient $C_D$.

\paragraph{Spearman Rank Correlation.}
For aerodynamic design tasks, accurate ranking of candidate geometries is often as important as absolute prediction error. Therefore, for AirfRANS and ShapeNet-Car, we report Spearman's rank correlation between predicted and ground-truth coefficients:
\[
    \rho
    =
    \frac{
    \operatorname{cov}(R(C),R(\hat{C}))
    }{
    \sigma_{R(C)}\sigma_{R(\hat{C})}
    },
\]
where $R(\cdot)$ denotes the rank transformation. Higher values indicate better preservation of design ordering.

\subsection{Implementation Details}
\label{app:implementations}

All experiments are conducted under consistent training protocols to ensure fair comparison across models and datasets. The full training and model configurations are summarized in Table~\ref{tab:hyperparameters}.

\textbf{Model Configuration.}
Unless otherwise specified, Advectra uses \(L=8\) layers, \(H_r=8\) routing heads, \(G=32\) latent anchors, and channel dimension \(D=128\). This setting is used across all benchmarks. For baseline models, we use the default best-reported configurations from the corresponding prior work whenever available.

\textbf{Transport, Kernel, and Ordering Implementation.}
The embedded token chart \(X=\phi_x(H_{\mathrm{phys}})\) and feature payload \(V=\phi_v(H_{\mathrm{phys}})\) are implemented by lightweight feedforward projections. For structured grid inputs, we use convolutional projections; for unstructured inputs, we use pointwise linear/MLP projections. For spatiotemporal benchmarks, the time offset is concatenated to the input token representation before these projections.

The anchor-frame context used by the transport maps is
\[
    s=\bar C\oplus\bar\Delta,
    \qquad
    \bar C=\frac{1}{G}\sum_{g=1}^{G}C_g,
    \qquad
    \bar\Delta=\frac{1}{G}\sum_{g=1}^{G}\Delta_g,
\]
computed independently for each head. The source anchors \(C\) are learned parameters, while the target offset \(\Delta\) is initialized at zero and parameterized with a bounded \(\tanh\) scaling in implementation. This makes \(\Delta\) a relative source--target displacement rather than an independent target anchor table.

The displacement maps \(\mathcal{T}\) and \(\mathcal{U}\) are implemented as GELU MLPs that take the current embedded coordinate together with the anchor-frame context. This parameterization makes the displacement fields smooth functions of their inputs for fixed network weights. We use the scaled residual form
\[
    \tau(x)=x+\epsilon\widetilde{\mathcal{T}}(x,s),
    \qquad
    \epsilon=0.95,
\]
and analogously for anchors. The unscaled networks \(\widetilde{\mathcal{T}}\) and \(\widetilde{\mathcal{U}}\) are constrained to be at most \(1\)-Lipschitz by construction. Therefore, the actual displacements \(\epsilon\widetilde{\mathcal{T}}\) and \(\epsilon\widetilde{\mathcal{U}}\) have Lipschitz constant at most \(0.95<1\), making the smooth residual coordinate updates locally non-folding by the inverse function theorem. In the main text, we absorb \(\epsilon\) into \(\mathcal{T}\) and write \(\tau(x)=x+\mathcal{T}(x,s)\).

The similarity kernel \(\kappa\) is implemented by a lightweight GELU MLP-based hyperspherical support function, whose hidden width is fixed relative to the model channel dimension. Given embedded token coordinates and transported anchors, both are normalized before computing assignment logits. We use a learned support module to modulate the local radius or diagonal metric of the kernel, and then obtain the source and target assignment matrices by softmax normalization over anchors. This implements the positive kernel \(\kappa:\mathbb{R}^{d}\times\mathbb{R}^{d}\to\mathbb{R}_{+}\) used in Sec.~\ref{sec:method}.

For geometric ordering, we compute midpoint anchors \(m_g=(C^{src}_g+C^{tgt}_g)/2\) and apply a deterministic greedy nearest-neighbor traversal in the transported anchor space. The traversal starts from the anchor with minimal coordinate under a fixed lexicographic ordering, and ties are broken deterministically. The ordering is refreshed every \(K\) iterations, with \(K=100\) unless otherwise specified, and no gradient is backpropagated through the discrete ordering operation. The ordered sequence is processed by a bidirectional SSM block.

\textbf{Training Protocols.}
For the industrial and engineering benchmarks (Airfoil, AirfRANS, ShapeNet-Car, and AirCraft), all models are trained using identical optimization settings as summarized in Table~\ref{tab:hyperparameters}. These include consistent choices of optimizer, learning rate, batch size, and scheduler to isolate architectural differences.

For the POSEIDON benchmarks (GCE-RT, NS-Tracer-PwC, and FNS-KF), we follow the semigroup training protocol introduced in POSEIDON. Specifically, we use the AdamW optimizer, batch size $32$, and a OneCycle learning rate scheduler. All models are trained to predict a fixed time horizon ($t=7$ for GCE-RT and $t=14$ for NS-Tracer-PwC and FNS-KF). Compared to the original POSEIDON setup, we use $1000$ training trajectories and $130$ testing trajectories, and apply spatial and temporal downsampling by a factor of $2$ for NS-Tracer-PwC and FNS-KF.


\paragraph{Parameter-Controlled Evaluation.}
For the industrial benchmarks AirfRANS, ShapeNet-Car, and AirCraft, we use the default best-reported configurations from prior work for all baselines. In particular, latent-token baselines such as Transolver and LaMO use the commonly reported \(G=32\) setting with channel width \(D=256\). Advectra uses the same latent-token budget \(G=32\) but a smaller channel width \(D=128\), resulting in comparable or lower parameter count. Thus, the main industrial comparisons do not reduce baseline token resolution relative to their reported default settings.

\textbf{Baselines and Reproduction Protocol.}
For the primary baselines, we follow official implementations and reported hyperparameter settings, including the default best-reported token counts and channel widths on industrial benchmarks when available. For datasets where results are not available, we retrain all models under these reported configurations and identical optimization settings.

\paragraph{Reproducibility and Compute.}
All experiments are conducted over three independent random seeds. We report the mean performance across runs for all metrics. Unless otherwise specified, experiments are run on a single-node setup with one NVIDIA RTX~4090 GPU. For larger experiments that exceed the memory capacity of this setup, such as AirCraft, we use a single-node setup with one NVIDIA RTX~A6000 GPU.

\textbf{AirCraft Settings and Baseline Adaptations.}
The AirCraft benchmark consists of high-fidelity 3D CFD simulations of aircraft geometries discretized as unstructured meshes with approximately $300{,}000$ nodes per sample. Each node contains spatial coordinates and geometric features, and the model predicts six physical quantities: pressure coefficient ($C_p$), density ($\rho$), velocity components ($U,V,W$), and static pressure ($P$).

We perform a grid search over architectural hyperparameters for all models, sweeping layers $\{2,4,6,8\}$, channels $\{64,128,256,512\}$, and heads $\{1,2,4,8\}$. The best-performing configuration is selected under a strict memory budget of $48$GB.

\begin{itemize}
    \item \textit{Baselines Implementation:} GraphSAGE, PointNet, and Graph U-Net are implemented following the Transolver codebase, while GNO, GINO, and related models are adapted from official implementations. All models are trained to jointly predict the six physical fields.

    \item \textit{Efficiency Considerations:} Transformer-based models (Transolver, LaMO, and Advectra) typically require less than $12$GB memory. Graph-based models such as GNO and MeshGraphNet exceed the $48$GB budget even with reduced hidden dimensions; for these baselines, we allow higher memory usage to enable comparison.
\end{itemize}

The training configuration for AirCraft follows the same setup as AirfRANS.

\begin{table}[H]
\caption{Training and model configurations for Advectra. Baseline configurations follow official or best-reported settings when available, as described in Appendix~\ref{app:implementations}. Relative L2 denotes field reconstruction error. \(L_v\) and \(L_s\) denote volume and surface losses. Here \(H_r\) denotes the number of parallel latent routing heads, i.e., independent anchor groups used for source/target assignments and feature projections. It is not a multi-head self-attention parameter; the latent sequence core is a bidirectional SSM.}
\label{tab:hyperparameters}
\begin{center}
\begin{small}
\begin{sc}
\resizebox{\textwidth}{!}{
\begin{tabular}{lcccccc|cccc}
\toprule
 & \multicolumn{6}{c}{\textbf{Training Configuration}} & \multicolumn{4}{c}{\textbf{Model Configuration}} \\
\cmidrule(lr){2-7} \cmidrule(lr){8-11}
\textbf{Benchmark} & \textbf{Loss} & \textbf{Epochs} & \textbf{Init LR} & \textbf{Optimizer} & \textbf{Batch} & \textbf{Scheduler} & $L$ & $H_r$ & $D$ & $G$ \\
\midrule
Airfoil & Relative L2 & 500 & $10^{-3}$ & AdamW & 4 & OneCycle & 8 & 8 & 128 & 32 \\
\midrule
AirfRANS & $L_v + L_s$ & 400 & $10^{-3}$ & Adam & 1 & OneCycle & 8 & 8 & 128 & 32 \\
ShapeNet-Car & $L_v + 0.5L_s$ & 200 & $10^{-3}$ & Adam & 1 & OneCycle & 8 & 8 & 128 & 32 \\
AirCraft & $L_v + L_s$ & 400 & $10^{-3}$ & Adam & 1 & OneCycle & 8 & 8 & 128 & 32 \\
\midrule
GCE-RT & Relative L2 & 500 & $10^{-3}$ & AdamW & 32 & OneCycle & 8 & 8 & 128 & 32 \\
NS-Tracer-PwC & Relative L2 & 500 & $10^{-3}$ & AdamW & 32 & OneCycle & 8 & 8 & 128 & 32 \\
FNS-KF & Relative L2 & 500 & $10^{-3}$ & AdamW & 32 & OneCycle & 8 & 8 & 128 & 32 \\
\bottomrule
\end{tabular}
}
\end{sc}
\end{small}
\end{center}
\end{table}

\subsection{Experimental Settings for Out-of-Distribution Generalization}

We evaluate generalization on unseen Reynolds numbers and Angles of Attack using AirfRANS. The training and test sets are constructed with disjoint parameter ranges. The detailed settings are provided in Table~\ref{tab:ood_settings}.

\begin{table}[h]
\centering
\caption{Settings for OOD generalization experiments on AirfRANS. The training and test sets contain disjoint ranges for Reynolds numbers and angles of attack.}
\label{tab:ood_settings}
\begin{small}
\begin{sc}
\resizebox{\textwidth}{!}{
\begin{tabular}{lccccc}
\toprule
& \multicolumn{2}{c}{\textbf{OOD Reynolds}} & \multicolumn{2}{c}{\textbf{OOD Angles}} \\
\cmidrule(lr){2-3} \cmidrule(lr){4-5}
\textbf{Dataset} & \textbf{Range} & \textbf{Samples} & \textbf{Range} & \textbf{Samples} \\
\midrule
Training Set & $[3 \times 10^6, 5 \times 10^6]$ & 500 & $[-2.5^\circ, 12.5^\circ]$ & 800 \\
Test Set     & $[2 \times 10^6, 3 \times 10^6] \cup [5 \times 10^6, 6 \times 10^6]$ & 500 & $[-5^\circ, -2.5^\circ] \cup [12.5^\circ, 15^\circ]$ & 200 \\
\bottomrule
\end{tabular}
}
\end{sc}
\end{small}
\end{table}

\section{Standard Deviation}
\begin{table}[H]
\caption{Standard deviations for GCE-RT, NS-Tracer, and FNS-KF benchmarks. We report the mean $\pm$ standard deviation over three independent runs. All values are relative errors. The best baseline model for each task is indicated in parentheses.}
\label{tab:std_dev_fluids}
\begin{center}
\begin{small}
\begin{sc}
\resizebox{\textwidth}{!}{
\begin{tabular}{lccccccc}
\toprule
& \multicolumn{4}{c}{GCE-RT} & \multicolumn{2}{c}{NS-Tracer-PwC} & FNS-KF \\
\cmidrule(lr){2-5} \cmidrule(lr){6-7} \cmidrule(lr){8-8}
Model & $\rho$ & $\mathbf{v}$ & $p$ & tracer & $\mathbf{v}$ & tracer & $\mathbf{v}$ \\
\midrule
Best Baseline & $0.0171{\scriptstyle\pm 0.0008}$ & $0.0671{\scriptstyle\pm 0.0012}$ & $0.0015{\scriptstyle\pm 0.0001}$ & $0.0551{\scriptstyle\pm 0.0014}$ & $0.1577{\scriptstyle\pm 0.0021}$ & $0.2047{\scriptstyle\pm 0.0068}$ & $0.1317{\scriptstyle\pm 0.0010}$ \\
& (Transolver++) & (Transolver++) & (LinearNO) & (Transolver++) & (FNO) & (U-Net) & (FNO) \\
\midrule
\textbf{Advectra (Ours)} & $\mathbf{0.0149}{\scriptstyle\pm 0.0003}$ & $\mathbf{0.0595}{\scriptstyle\pm 0.0007}$ & $\mathbf{0.0014}{\scriptstyle\pm 0.0001}$ & $\mathbf{0.0488}{\scriptstyle\pm 0.0012}$ & $\mathbf{0.1209}{\scriptstyle\pm 0.0011}$ & $\mathbf{0.1607}{\scriptstyle\pm 0.0025}$ & $\mathbf{0.1202}{\scriptstyle\pm 0.0014}$ \\
\bottomrule
\end{tabular}
}
\end{sc}
\end{small}
\end{center}
\end{table}

\begin{table}[H]
\caption{Standard deviations for geometric and mesh-based benchmarks. We report the mean $\pm$ standard deviation over three independent runs. Spearman's rank correlations for ShapeNet-Car and AirfRANS are presented as percentages (scaled by 100). Airfoil values are relative errors. The best baseline model for each task is indicated in parentheses.}
\label{tab:std_dev_mesh}
\begin{center}
\begin{small}
\begin{sc}
\resizebox{0.5\textwidth}{!}{
\begin{tabular}{lccc}
\toprule
& Airfoil & ShapeNet-Car & AirfRANS \\
Model & (Rel. Error) & (Corr. \%) & (Corr. \%) \\
\midrule
Best Baseline & $0.0047{\scriptstyle\pm 0.0001}$ & $99.28{\scriptstyle\pm 0.02}$ & $99.92{\scriptstyle\pm 0.01}$ \\
& (LaMO) & (SpiderSolver) & (LinearNO) \\
\midrule
\textbf{Advectra (Ours)} & $\mathbf{0.0041}{\scriptstyle\pm 0.0001}$ & $\mathbf{99.41}{\scriptstyle\pm 0.02}$ & $\mathbf{99.95}{\scriptstyle\pm 0.01}$ \\
\bottomrule
\end{tabular}
}
\end{sc}
\end{small}
\end{center}
\end{table}

\section{Limitations, Future Work, and Societal Impact}

This work focuses on transport-dominated systems in which coherent structures move and deform across coordinate frames. Advectra uses learned, locally non-folding displacement maps to define approximately co-moving latent source and target charts. This constraint applies to the representation-level coordinate map, not to the full input-output operator. Therefore, it does not preclude modeling sharp geometric boundaries, discontinuous coefficients, or localized field structures, which can still be represented through feature channels, learned assignments, nonlinear latent dynamics, and reconstruction. However, strongly shock-dominated regimes with characteristic crossing, topology change, or many-to-one physical transport may require additional mechanisms beyond a single smooth latent chart. We leave further investigation of piecewise latent charts, shock-aware gating, conservative reconstruction, and discontinuity-aware transport parameterizations as future work.

The learned maps are not guaranteed to recover identifiable physical trajectories or preserve global measures exactly. They should instead be viewed as architectural mechanisms for aligning evolving structures in latent space. Our analysis and ablations support the roles of transport, ordering, sorting frequency, and token count in the evaluated regimes, while exhaustive sweeps over all hyperparameters and datasets remain computationally expensive. Future work could further strengthen the physical structure of the learned maps through conservative or geometry-preserving parameterizations.

\paragraph{Societal impact.}
Advectra is a general surrogate modeling method for physics and engineering simulations, so its societal effects are primarily indirect. Potential benefits include reducing the computational cost of scientific simulation and accelerating design workflows. Potential risks arise if learned surrogates are used as certified solvers in safety-critical settings without adequate validation, uncertainty estimation, or domain-specific verification. Privacy and fairness concerns are not central to the PDE and engineering benchmarks studied here.

\clearpage

\end{document}